\documentclass{article}

\pdfoutput=1

\usepackage{commands}

\title{Deep Sketched Output Kernel Regression\\for Structured Prediction}

\author{%
  Tamim El Ahmad$^*$ \\
  LTCI, Télécom Paris\\
  IP Paris\\
  France \\
  \texttt{tamim.elahmad@telecom-paris.fr} \\
  \And
  Junjie Yang$^*$ \\
  LTCI, Télécom Paris\\
  IP Paris\\
  France \\
  \texttt{junjie.yang@telecom-paris.fr} \\
  \AND
  Pierre Laforgue \\
  Department of Computer Science \\
  University of Milan \\
  Italy \\
  \texttt{pierre.laforgue@unimi.it} \\
  \And
  Florence d'Alch\'e-Buc \\
  LTCI, Télécom Paris\\
  IP Paris\\
  France \\
  \texttt{florence.dalche@telecom-paris.fr} \\
}

\begin{document}

\maketitle

\begin{abstract}
By leveraging the kernel trick in the output space, kernel-induced losses provide a principled way to define structured output prediction tasks for a wide variety of output modalities.
In particular, they have been successfully used in the context of surrogate non-parametric regression, where the kernel trick is typically exploited in the input space as well.
However, when inputs are images or texts, more expressive models such as deep neural networks seem more suited than non-parametric methods.
In this work, we tackle the question of how to train neural networks to solve structured output prediction tasks, while still benefiting from the versatility and relevance of kernel-induced losses.
We design a novel family of deep neural architectures, whose last layer predicts in a data-dependent finite-dimensional subspace of the infinite-dimensional output feature space deriving from the kernel-induced loss.
This subspace is chosen as the span of the eigenfunctions of a randomly-approximated version of the empirical kernel covariance operator.
Interestingly, this approach unlocks the use of gradient descent algorithms (and consequently of any neural architecture) for structured prediction.
Experiments on synthetic tasks as well as real-world supervised graph prediction problems show the relevance of our method. \footnotetext[1]{Equal contribution.}
\end{abstract}

\section{Introduction}
\label{sec:intro}

%





Learning to predict complex outputs, such as graphs or any other composite object, raises many challenges in machine learning \citep{bakir2007, nowozin2011, deshwal2019struct}.
The most important of them is undoubtedly the difficulty of leveraging the geometry of the output space.
%
In supervised graph prediction, for instance, it is often required to use node permutation-invariant and node size-insensitive distances, such as the Fused Gromov-Wasserstein distance \citep{vayer_optimal_2019}.
In that regard, surrogate methods such as Output Kernel Regression \citep{westonNIPS2002, geurts2006, Kadri_icml2013} offer a powerful and flexible framework by using the kernel trick in the output space.
By appropriately choosing the output kernel, it is possible to incorporate various kinds of information, both in the model and in the loss function \citep{nowak19a, ciliberto2020general, cabannes2021fast}.
One important limitation of this approach, however, is that the induced output features may be infinite-dimensional.

If leveraging the kernel trick in the input space may be a solution \citep{cortes2005, brouard2016input}, such non-parametric methods are usually outperformed by more expressive models such as neural networks when input data consist of images or texts.
In the context of structured prediction, deep learning has led to impressive results for specific tasks, such as semantic segmentation \citep{kirillov2023segment} or the protein 3D structure prediction \citep{jumper2021highly}.
To create versatile deep models, the main approach explored in the literature is the energy-based approach, which consists of converting structured prediction into learning a scalar score function \citep{LeCun2006ATO, belanger2016structured, Gygli2017, lee2022structured}.
However, these methods usually fail to go beyond structured prediction problems which can be reformulated as high-dimensional multi-label classification problems, as pointed out by \citet{graber2018deep}.
%
Besides, this approach requires a two-step strategy, since the energy function is first learned thanks to the training data, and then maximized at inference time.
To obtain an end-to-end model, \citet{belanger2017endtoend} uses direct risk minimization techniques, and \citet{tu2018learning} introduces inference networks, a neural architecture that approximates the inference problem.
In this work, we choose to benefit from the versatility of kernel-induced losses, and deploy it to neural networks.
To this end, we address the infinite-dimensionality of the output features by computing a finite-dimensional basis within the output feature space, defined as the eigenbasis of a sketched version of the output empirical covariance operator.

Sketching \citep{mahoney2011randomized, woodruff14} is a dimension-reduction technique based on random linear projections.
%
In the context of kernel methods, it has mainly been explored through the so-called Nystr\"om approximation \citep{WilliamsNystromNIPS2000, rudi2015less}, or via specific distributions such as Gaussian or Randomized Orthogonal Systems \citep{Yang2017, lacotte2022adaptive}.
Previous works tackle sketched scalar kernel regression by providing a low-rank approximation of the Gram matrix \citep{drineas2005, bach2013sharp}, reducing the number of parameters to learn at the optimization stage \citep{Yang2017, lacotte2022adaptive}, providing data-dependent random features \citep{WilliamsNystromNIPS2000, Yang_NIPS2012_621bf66d, Kpotufe2020}, or leveraging an orthogonal projection operator in the feature space \citep{rudi2015less}.
This last interpretation has been used to learn large-scale dynamical systems \citep{meanti2023estimating}, and structured prediction \citep{elahmad2024sketch}.

In our proposition to solve structured prediction from complex input data, we make the following contributions:
\begin{itemize}[itemsep=5pt]
    \item[$\bullet$] We introduce Deep Sketched Output Kernel Regression, a novel family of deep neural architectures whose last layer predicts a data-dependent finite-dimensional representation of the outputs, that lies in the infinite-dimensional feature space deriving from the kernel-induced loss.
    \item[$\bullet$] This last layer is computed beforehand, and is the eigenbasis of the sketched empirical covariance operator, unlocking the use of gradient-based techniques to learn the weights of the previous layers for any neural architecture.
    \item[$\bullet$] We empirically show the relevance of our approach on a synthetic least squares regression problem, and provide a strategy to select the sketching size.
    \item[$\bullet$] We show that DSOKR performs well on two text-to-molecule datasets.
\end{itemize}

%
%
%
%
%
%
%


\section{Deep Sketched Output Kernel Regression}
\label{sec:DSOKR}

In this section, we set up the problem of structured prediction.
Specifically, we consider surrogate regression approaches for kernel-induced losses.
By introducing a last layer able to make predictions in a Reproducing Kernel Hilbert Space (RKHS), we unlock the use of deep neural networks as hypothesis space.

Consider the general regression task from an input domain $\bmX$ to a structured output domain $\bmY$ (e.g., the set of labeled graphs of arbitrary size).
%
%
%
Learning a mapping from $\bmX$ to $\bmY$ naturally requires taking into account the structure of the output space.
%
One way to do so is the {\it Output Kernel Regression} (OKR) framework \citep{westonNIPS2002, cortes2005, geurts2006, Brouard_icml11, brouard2016input},
which is part of the family of surrogate regression methods \citep{Ciliberto2016,ciliberto2020general}.

\paragraph{\bf Output Kernel Regression.}
A positive definite (p.d.) kernel $\kernely: \bmY \times \bmY \to \reals$ is a symmetric function such that for all $n \geq 1$, and any $\left(y_i\right)_{i=1}^n \in \bmY^n$, $\left(\alpha_i\right)_{i=1}^n \in \reals^n$, we have $\sum_{i,j=1}^n \alpha_i\, \kernely\left(y_i, y_j\right)\alpha_j \geq 0$.
Such a kernel is associated with a canonical feature map $\psiy \colon\! y \in \bmY \mapsto \kernely(\cdot, y)$, which is uniquely associated with a Hilbert space of functions $\bmH \subset \reals^{\bmY}$, the RKHS, such that $\psiy(y) \in \bmH$ for all $y \in \bmY$, and $h\left(y\right) = \langle h, \psiy(y)\rangle_{\bmH}$ for any $\left(h, y\right) \in \bmH \times \bmY$.
Given a p.d. kernel $\kernely$, $\psiy$ its canonical feature map and $\Hy$ its RKHS,
the OKR approach that we consider in this work exploits the kernel-induced squared loss:
\looseness-1
\begin{equation}\label{eq:loss}
\Delta(y, y^\prime) \coloneqq \| \psiy(y) - \psiy(y^\prime) \|_{\Hy}^2 = \kernely(y, y) - 2 \kernely(y, y^\prime) + \kernely(y^\prime, y^\prime)\,.
\end{equation}
%

The versatility of loss \eqref{eq:loss} stems from the large variety of kernels that have been designed to compare structured objects \citep{Gartner08, Korba18, borgwardt20}.
In multi-label classification, for instance, choosing the linear kernel or the Tanimoto kernel induces respectively the Hamming and the F1-loss \citep{tanimoto1958elementary}.
In label ranking, Kemeny and Hamming embeddings define Kendall’s $\tau$ distance and the Hamming loss \citep{Korba18, nowak20a} respectively.
%
%
%
For sequence prediction tasks, n-gram kernels have been proven useful \citep{cortes2007, Kadri_icml2013, nowak20a}, while an abundant collection of kernels has been designed for graphs, 
based either on bags of structures or information propagation, see \cref{apx:graph_pred} and \citet{borgwardt20} for examples.
%

If kernel-induced losses can be computed easily thanks to the kernel trick, note that most of them are however non-differentiable.
%
%
%
%
%
%
In particular, this largely compromises their use within deep neural architectures, that are however key to achieve state-of-the-art performances in many applications.
In this work, we close this gap and propose an approach that benefits from both the expressivity of neural networks for input image/textual data, as well as the relevance of kernel-induced losses for structured outputs.
Formally, let $\rho$ be a joint probability distribution on $\bmX \times \bmY$.
Our goal is to design a family $(f_\theta)_{\theta \in \Theta} \subset \bmY^\bmX$ of neural networks with outputs in $\bmY$  that can minimize the kernel-induced loss, i.e., that can solve
\looseness-1
\begin{equation}\label{eq:pb-init}
    \underset{\theta \in \Theta}{\min} ~ \mathbbm{E}_{(x, y) \sim \rho}\Big[\left\|\psiy(y) - \psiy\big(f_\theta(x)\big)\right\|_{\Hy}^2\Big]\,.
\end{equation}
\begin{figure}[t]
\centering
\includegraphics[width=0.6\textwidth]{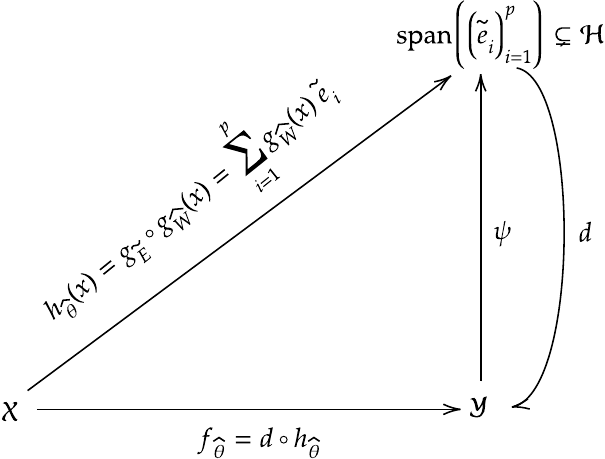}
\caption{Illustration of DSOKR model.} \label{fig:DSOKR}
\end{figure}
To do so, we assume that we can access a training sample $\{(x_1,y_1), \ldots,(x_n,y_n)\}$ drawn i.i.d. from $\rho$.
%
Since learning $f_\theta$ through $\psiy$ is difficult, we employ a two-step method. 
First, we solve the surrogate empirical problem
\begin{equation}\label{eq:init_loss}
    \hat{\theta} \in \argmin_{\theta \in \Theta} ~ L(\theta) = \argmin_{\theta \in \Theta}~\frac{1}{n} \sum_{i=1}^n \|h_\theta(x) - \psiy(y)\|_{\Hy}^2\,,
\end{equation}
where $(h_\theta)_{\theta \in \Theta} \subset \bmH^\bmX$ is a family of neural networks with outputs in $\bmH$. 
We then retrieve the solution by solving for any prediction the pre-image problem 
\begin{equation}
    f_{\hat{\theta}}(x) = \argmin_{y \in \bmY} ~ \|h_{\hat{\theta}}(x) - \psiy(y)\|_{\Hy}^2\,.
\end{equation}
This approach nonetheless raises a major challenge.
Indeed, the dimension of the canonical feature space $\Hy$ may be infinite, making the training very difficult.
The question we have to answer now is: {\it how can we design a neural architecture that is able to learn infinite-dimensional output kernel features?}

\paragraph{\bf Neural networks with infinite-dimensional outputs.}
We propose a novel architecture of neural networks to compute the function $h_{\theta}$ with values in~$\Hy$.
Let $p \geq 1$, our architecture is the composition of two networks: an input neural network, denoted $g_{W}\colon \bmX \to \reals^p$, with generic weights $W \in \bmW$, and a last layer composed of a unique {\it functional} neuron, denoted $g_{\basis}\colon \reals^p \to \Hy$, that predicts in $\Hy$.
The latter depends on the kernel $\kernely$ used in the loss definition, and on a finite basis $\basis=((e_j)_{j=1}^p)\in \Hy^p$ of elements in $\Hy$.
We let $\theta=(W, \basis)$, and for any $x \in \bmX$, we have
\looseness-1
\begin{equation}
    h_{\theta}(x) \coloneqq g_{\basis} \circ g_{W}(x)\,,
\end{equation}
where $g_{W}$ typically implements a $L-1$ neural architecture encompassing, multilayered perceptrons, convolutional neural networks, or transformers.
Instead, $g_{\basis}$ computes a linear combination of some basis functions $\basis=(e_j)_{j=1}^p \in \Hy^p$
\begin{equation}
    g_{\basis} : z \in \reals^{\ranksketchGramY} \mapsto \sum_{j=1}^p z_j e_j \in \Hy\,.
\end{equation}
With this architecture, computations remain finite, and the input neural network outputs the coefficients of the basis expansion, generating predictions in $\Hy$.
\begin{remark}[Input Neural net's last layers]
    Since the neural network $g_{W}$ learns the coordinates of the surrogate estimator in the basis, its last layers are always mere fully connected ones, regardless of the nature of the output data at hand.
\end{remark}

\subsection{Learning neural networks with infinite-dimensional outputs}

Learning the surrogate regression model $h_\theta$ now boils down to computing $\theta = (W, \basis)$.
We propose to solve this problem in two steps.
First, we learn a suitable $\basis$ using only the output training data $(\psi(y_i))_{i=1}^n$ in an unsupervised fashion.
Then, we use standard gradient-based algorithms to learn $W$ through the frozen last layer, minimizing the loss on the whole supervised training sample $(x_i,\psi(y_i))_{i=1}^n$.

\paragraph{\bf Estimating the functional last unit $\bm{g_{\basis}}$.}
A very first idea is to choose $\basis$ as the non-orthogonal dictionary $\psi(y_j)_{j=1}^n$.
But this choice induces a very large output dimension (namely, $p=n$) for large training datasets.

An alternative consists in using Kernel Principal Component Analysis (KPCA) \citep{scholkopf1997kernel}.
Given a marginal probability distribution over $\bmY$, let $\covy = \E_y[\psiy(y) \otimes \psiy(y)]$ be the covariance operator associated with $\kernely$, and $\empcovy = (1/n) \sum_{i=1}^n \psiy(y_i) \otimes \psiy(y_i)$ its empirical counterpart.
Let $\SY$ be the sampling operator that transforms a function $f \in \Hy$ into the vector $(1/\sqrt{n}) (f(x_1), \ldots ,f(x_n))^\top$ in $\reals^n$, and denote by $\SY^{\#}$ its adjoint.
We have $\SY^{\#}\colon \alpha \in \reals^n \mapsto (1/\sqrt{n}) \sum_{i=1}^n \alpha_i \psiy(y_i) \in \Hy$, and $\empcovy = \SY^\# \SY$.
KPCA provides the eigenbasis of $\empcovy$ by computing the SVD of the output Gram matrix, for a prohibitive computational cost of $\bmO(n^3)$.
In practice, though, it is often the case that the so-called {\it capacity condition} holds \citep{ciliberto2020general, elahmad2024sketch}, i.e., that the spectrum of the empirical covariance operator enjoys a large eigendecay.
It is then possible to efficiently approximate the eigenbasis of $\empcovy$ using random projections techniques \citep{mahoney2011randomized}, also known as sketching, solving this way the computational and memory issues.

\paragraph{\bf Sketching for kernel methods.}
Sketching \citep{woodruff14} is a dimension reduction technique based on random linear projections.
Since the goal is to reduce the dependency on the number of training samples $n$ in kernel methods, such linear projections can be encoded by a randomly drawn matrix $\sketchy \in \reals^{\my \times n}$, where $\my \ll n$.
Standard examples include Nystr\"{o}m approximation \citep{meanti2020kernel}, where each row of $\sketchy$ is randomly drawn from the rows of the identity matrix $I_n$, also called sub-sampling sketches, and Gaussian sketches \citep{Yang2017}, where all entries of $R$ are i.i.d. Gaussian random variables.
As they act as a random training data sub-sampler and then largely reduce both the time and space complexities induced by kernel methods, sub-sampling sketches are the most popular sketching type applied to kernels, while Gaussian sketches are less computationally efficient but offer better statistical properties.
Hence, given a sketching matrix $\sketchy \in \reals^{\my \times n}$, one can defines $\tilde{\bmH}_{\bmY} = \spn((\sum_{j=1}^n \sketchyij \psiy(y_j))_{i=1}^{\my})$ which is a low-dimensional linear subspace of $\Hy$ of dimension at most $\my$.
One can even compute the basis $\tilde{E}$ of $\tilde{\bmH}_{\bmY}$, providing the last layer $g_{\tilde{E}}$.

\paragraph{\bf Sketching to estimate $\bm{g_E}$.}
We here show how to compute the basis $\tilde{E}$ of $\tilde{\bmH}_{\bmY}$.
Let $\my < n$, and $\sketchy \in \reals^{\my \times n}$ be a sketching matrix.
Let $\sketchGramY = \sketchy \GramY \sketchy^\top \in \reals^{\my \times \my}$ be the sketched Gram matrix, and $\big\{(\sigma_i(\sketchGramY), \tilde{\mathbf{v}}_{i}), i \in[\my]\big\}$ its eigenpairs, in descending order.
We set $\ranksketchGramY = \operatorname{rank}\big(\sketchGramY\big)$.
Note that $p \leq \my$, and that $p = \my$ for classical examples, e.g. full-rank $\GramY$ and sub-sample without replacement or Gaussian $\sketchy$.
%
%
The following proposition provides the eigenfunctions of the sketched empirical covariance operator.
\begin{proposition}\label{prop:proj_exp}\citep[Proposition~2]{elahmad2024sketch}
    The eigenfunctions of the sketched empirical covariance operator $\sketchempcovy = \SY^{\#}\!\sketchy^\top\! \sketchy \SY$ are the $\tilde{e}_j = \sqrt{\frac{n}{\sigma_j(\sketchGramY)}} \SY^{\#} \sketchy^\top \tilde{\mathbf{v}}_{j} \in \Hy$, for $j \leq \ranksketchGramY$.
\end{proposition}
Hence, computing the eigenfunctions of $\sketchempcovy$ provides a basis of $\Hy$ of dimension $p$. 
%
Note that in sketched KPCA, which has been explored via Nystr\"om approximation in \citet{stergeNKPCA20a, stergeNKPCA22}, one solves for $i = 1, \ldots, \my$
\begin{equation}\label{eq:SKPCA}
    f_i = \argmax_{f \in \Hy} ~ \left\{ \langle f, \empcovy f \rangle_{\Hy} : f \in \tilde{\bmH}_{\bmY}, \left\|f\right\|_{\Hy} = 1, f \perp \{f_1, \ldots, f_{i-1}\} \right\}
\end{equation}
where $\tilde{\bmH}_{\bmY} = \spn((\sum_{j=1}^n \sketchyij \psiy(y_j))_{i=1}^{\my})$.
Let $\projY$ be the orthogonal projector onto the basis $(\tilde{e}_1, \ldots, \tilde{e}_{\ranksketchGramY})$, solving \cref{eq:SKPCA} is equivalent to compute the eigenfunctions of the projected empirical covariance operator $\projY \empcovy \projY$, i.e., to compute the KPCA of the projected kernel $\langle \projY \psiy(\cdot), \projY \psiy(\cdot) \rangle_{\Hy}$.
Besides, as for the SVD of $\sketchempcovy$, sketched KPCA needs the SVD of $\sketchGramY$ to obtain its square root, but also requires the additional $\sketchGramY^{1/2} \sketchy \GramY^2 \sketchy^\top \sketchGramY^{1/2}$ SVD computation.


\begin{remark}[Random Fourier Features]
    Another popular kernel approximation is the Random Fourier Features \citep{rahimi2007, rudi2017generalization, li2021unified}.
    They approximate a kernel function as the inner product of small random features using Monte-Carlo sampling when the kernel writes as the Fourier transform of a probability distribution.
    Such an approach, however, defines a new randomly approximated kernel, then a new randomly approximated loss, which can induce learning difficulties due to the bias and variance inherent to the approximation.
    Unlike RFF, sketching is not limited to kernels writing as the Fourier transform of a probability distribution and to defining an approximated loss, it allows the building of a low-dimensional basis within the original feature space of interest.
\end{remark}

\paragraph{\bf Learning the input neural network $\bm{g_{W}}$.}
Equipped with the basis $\tilde{E}=(\tilde{e}_j)_{j\le p}$, we can compute a novel expression of the loss $L(\theta)=L(\tilde{E}, W)$, see \cref{apx:proof} for the proof.
\begin{restatable}{proposition}{propnewloss}\label{prop:new-loss}
Given the pre-trained basis $\tilde{E}=(\tilde{e}_j)_{j\le p}$, $L(\tilde{E}, W)$ expresses as
\begin{equation}\label{eq:loss-exp}
    L(\tilde{E}, W)= \frac{1}{n} \sum_{i=1}^n \left\|g_W(x_i) - \sketchpsiy(y_i)\right\|_{2}^2\,,
\end{equation}
where $\sketchpsiy(y) = (\tilde{e}_1(y), \ldots, \tilde{e}_{\ranksketchGramY}(y))^\top = \widetilde{D}_{\ranksketchGramY}^{-1 / 2} \widetilde{V}_{\ranksketchGramY}^\top \sketchy \kernelvectY \in \reals^{\ranksketchGramY}$, $\widetilde{V}_{\ranksketchGramY}=(\tilde{\mathbf{v}}_{1}, \ldots, \tilde{\mathbf{v}}_{\ranksketchGramY})$, $\widetilde{D}_{\ranksketchGramY}=\operatorname{diag}(\sigma_1(\sketchGramY), \ldots, \sigma_{\ranksketchGramY}(\sketchGramY))$,  and $\kernelvectY = (\kernely(y, y_1), \ldots,\kernely(y, y_n))$.
\end{restatable}
%
Finally, given $\tilde{E}$ and Prop. \ref{prop:new-loss}, learning the full network $h_{\theta}$ boils down to learning the input neural network $g_{W}$ and thus finding a solution $\hat{W}$ to
\begin{equation}
    \min_{W \in \mathcal{W}} ~ \frac{1}{n} \sum_{i=1}^n \left\|g_W(x_i) - \sketchpsiy(y_i)\right\|_{2}^2\,.\label{eq:newloss}
\end{equation}
%
A classical stochastic gradient descent algorithm can then be applied to learn $W$.
Compared to the initial loss \eqref{eq:init_loss}, the relevance of \eqref{eq:newloss} is governed by the quality of the approximation of $\empcovy$ by $\sketchempcovy$.
%
%
If our approach regularises the solution (the range of the surrogate estimator $h_\theta$ is restricted from $\Hy$ to $\basis$), this restriction may not be limiting if we set $\my \geq p$ high enough to capture all the information contained in $\empcovy$.
We discuss strategies to correctly set $\my$ at the beginning of \cref{sec:xps}.

\begin{remark}[Beyond the square loss]
    Equipped with such an architecture $g_W \circ g_E$, one can easily consider any loss that writes $\Delta(y,y') = c(\|\psiy(y) - \psiy(y^\prime)\|^2_{\Hy})$, where $c : \reals_+ \rightarrow \reals_+$ is a non-decreasing sub-differentiable function.
    For instance, in the presence of output outliers, one could typically consider robust losses such as the Huber or $\epsilon$-insensitive losses, that correspond to different choices of function $c$ \citep{laforgue2020duality,huber1964robust,steinwart-epsilon}.
\end{remark}

\subsection{The pre-image problem at inference time} 

\begin{algorithm}[!t]
\caption{Deep Sketched Output Kernel Regression (DSOKR)}
\label{alg:dsokr}
\SetKwInOut{Input}{input}
\SetKwInOut{Init}{init}
\Input{training $\{(x_i, y_i)\}_{i=1}^n$, validation $\{(x_i^{\operatorname{val}}, y_i^{\operatorname{val}})\}_{i=1}^{n_{\operatorname{val}}}$ pairs, test inputs $\{x_i^{\operatorname{te}}\}_{i=1}^{n_{\operatorname{te}}}$, candidate outputs test inputs $\{y_i^{\operatorname{c}}\}_{i=1}^{n_{\operatorname{c}}}$, normalized output kernel $\kernely$, sketching matrix $\sketchy \in \reals^{\my \times n}$, neural network $g_W$}\vspace{0.2cm}
\Init{$\sketchGramY = \sketchy \GramY \sketchy^\top \in \reals^{\my \times \my}$ where $\GramY = (\kernely(y_i, y_j))_{1 \leq i, j \leq n} \in \reals^{n \times n}$} \vspace{0.2cm}

\tcp{1. a. Training of $g_E$: computations for the basis $\widetilde{E}$}

$\bullet$ Construct $\widetilde{D}_{\ranksketchGramY} \in \reals^{\ranksketchGramY \times \ranksketchGramY}$, $\widetilde{V}_{\ranksketchGramY} \in \reals^{\my \times \ranksketchGramY}$  such that $\widetilde{V}_{\ranksketchGramY} \widetilde{D}_{\ranksketchGramY} \widetilde{V}_{\ranksketchGramY}^\top = \sketchGramY$ (SVD of $\sketchGramY$)

$\bullet$ $\widetilde{\Omega} = \widetilde{D}_{\ranksketchGramY}^{-1 / 2} \widetilde{V}_{\ranksketchGramY}^\top \in \reals^{\ranksketchGramY \times \my}$ \vspace{0.2cm}

\tcp{1. b. Training of $g_W$: solving the surrogate problem}

$\bullet$ $\sketchpsiy(y_i) = \widetilde{\Omega} \sketchy k^{y_i} \in \reals^{\ranksketchGramY}, \forall ~ 1 \leq i \leq n$, $\sketchpsiy(y_i^{\operatorname{val}}) = \widetilde{\Omega} \sketchy k^{y_i^{\operatorname{val}}} \in \reals^{\ranksketchGramY}, \forall ~ 1 \leq i \leq n_{\operatorname{val}}$

$\bullet$ $\hat{W} = \underset{W \in \mathcal{W}}{\argmin} ~ \frac{1}{n} \sum_{i=1}^n \left\|g_W(x_i) - \sketchpsiy(y_i)\right\|_2^2$ (training of $g_W$ with training $\{(x_i, \sketchpsiy(y_i))\}_{i=1}^n$ and validation $\{(x_i^{\operatorname{val}}, \sketchpsiy(y_i^{\operatorname{val}}))\}_{i=1}^{n_{\operatorname{val}}}$ pairs and Mean Squared Error loss) \vspace{0.2cm}

\tcp{2. Inference}

$\bullet$ $\sketchpsiy(y_i^{\operatorname{c}}) = \widetilde{\Omega} \sketchy k^{y_i^{\operatorname{c}}} \in \reals^{\ranksketchGramY}, \forall ~ 1 \leq i \leq n_{\operatorname{c}}$

$\bullet$ $f_{\hat{\theta}}(x_i^{\operatorname{te}}) = y_j^{\operatorname{c}}$ where $j = \underset{1 \leq j \leq n_{\operatorname{c}}}{\argmax} ~ g_{\hat{W}}(x_i^{\operatorname{te}})^\top \tilde{\psi}(y_j^{\operatorname{c}})$, $\forall ~ 1 \leq i \leq n_{\operatorname{te}}$ \vspace{0.2cm}

\Return{$f_{\hat{\theta}}(x_i^{\operatorname{te}}), \forall ~ 1 \leq i \leq n_{\operatorname{te}}$}
\end{algorithm}

We focus now on the decoding part, i.e., on computing
\[
d \circ h_{\hat{\theta}}(x) = \argmin_{y \in \bmY} ~ \kernely(y, y) - 2 g_{\hat{W}}(x)^\top \sketchpsiy(y) = \argmax_{y \in \bmY} ~ g_{\hat{W}}(x)^\top \sketchpsiy(y)
%
\]
if we assume $\kernely$ to be normalized, i.e. $\kernely(y, y^\prime) = 1, \forall y, y^\prime \in \bmY$.
For a test set $X^{\operatorname{te}} = (x_1^{\operatorname{te}}, \ldots, x_{n_{\operatorname{te}}}^{\operatorname{te}}) \in \bmX^{n_{\operatorname{te}}}$ and a candidate set $Y^{\operatorname{c}} = (y_1^{\operatorname{c}}, \ldots, y_{n_{\operatorname{c}}}^{\operatorname{c}}) \in \bmY^{n_{\operatorname{c}}}$, for all $1 \leq i \leq n_{te}$, the prediction is given by
\begin{equation}
    f_{\hat{\theta}}(x_i^{\operatorname{te}}) = y_j^{\operatorname{c}} \quad \text{where} \quad j = \underset{1 \leq j \leq n_{\operatorname{c}}}{\argmax} ~ g_{\hat{W}}(x_i^{\operatorname{te}})^\top \tilde{\psi}(y_j^{\operatorname{c}})\,.
\end{equation}


Hence, the decoding is particularly suited to problems for which we have some knowledge of the possible outcomes, such as molecular identification problems \citep{Brouard_ismb2016}.
%
%
%
When the output kernel is differentiable, it may also be solved using standard gradient-based methods.
Finally, some ad-hoc ways to solve the pre-image problem exist for specific kernels, see e.g., \citet{cortes2007} for the sequence prediction via n-gram kernels, or \citet{Korba18} for label ranking via Kemeny, Hamming, or Lehmer embeddings.
The DSOKR framework is summarized in \cref{alg:dsokr}.

\section{Experiments}
\label{sec:xps}

In this section, we first present a range of strategies to select the sketching size and an analysis of our proposed DSOKR on a synthetic dataset.
Besides, we show the effectiveness of DSOKR through its application to two real-world Supervised Graph Prediction (SGP) tasks: SMILES to Molecule and Text to Molecule.
The code to reproduce our results is available at: \href{https://github.com/tamim-el/dsokr}{https://github.com/tamim-el/dsokr}.

%

\paragraph{\bf Sketching size selection strategy.} A critical hyper-parameter of DSOKR is the sketching size $\my$.
Indeed, the optimal choice is the dimension of the subspace containing the output features.
However, to estimate this dimension, one has to compute the eigenvalues of $\GramY$, which has the prohibitive complexity of $\bmO(n^3)$.
Hence, a first solution is to compute the Approximate Leverage Scores (ALS) as described in \citet{elalaoui_NIPS2015}.
This is an approximation of the eigenvalues of $\GramY$ that relies on sub-sampling $n_S < n$ entries within the whole training set.
%
%
Moreover, we use another technique that we call \textit{Perfect h}.
Considering any pair $(x, y)$ in a validation set, we replace $g_{W}(x)$ by the ``perfect'' coefficients of the expansion, i.e., for each $j=1, \ldots, p$, ~$\langle \tilde{e}_j, \psiy(y) \rangle_{\Hy}$ and define ``perfect'' surrogate estimator $h_{\psiy}$ as follows
\looseness-1
\begin{equation}
    h_{\psiy}(x) = \sum_{j=1}^{\ranksketchGramY} \langle \tilde{e}_j, \psiy(y) \rangle_{\Hy} ~ \tilde{e}_j = \sum_{j=1}^{\ranksketchGramY} \sketchpsiy(y)_j ~ \tilde{e}_j\,.
\end{equation}
Then, we evaluate the performance of this ``perfect'' surrogate estimator $h_{\psiy}$ on a validation set to select $\my$.
%
Hence, \textit{Perfect h} allows to select the minimal $\my$ in the range given by ALS such that the performance of $h_{\psiy}$ reaches an optimal value.
%




 \subsection{Analysis of DSOKR on Synthetic Least Squares Regression}
\label{subsec:synth_lsr}

\begin{figure}[!t]
\centering
\includegraphics[width=0.288\textwidth]{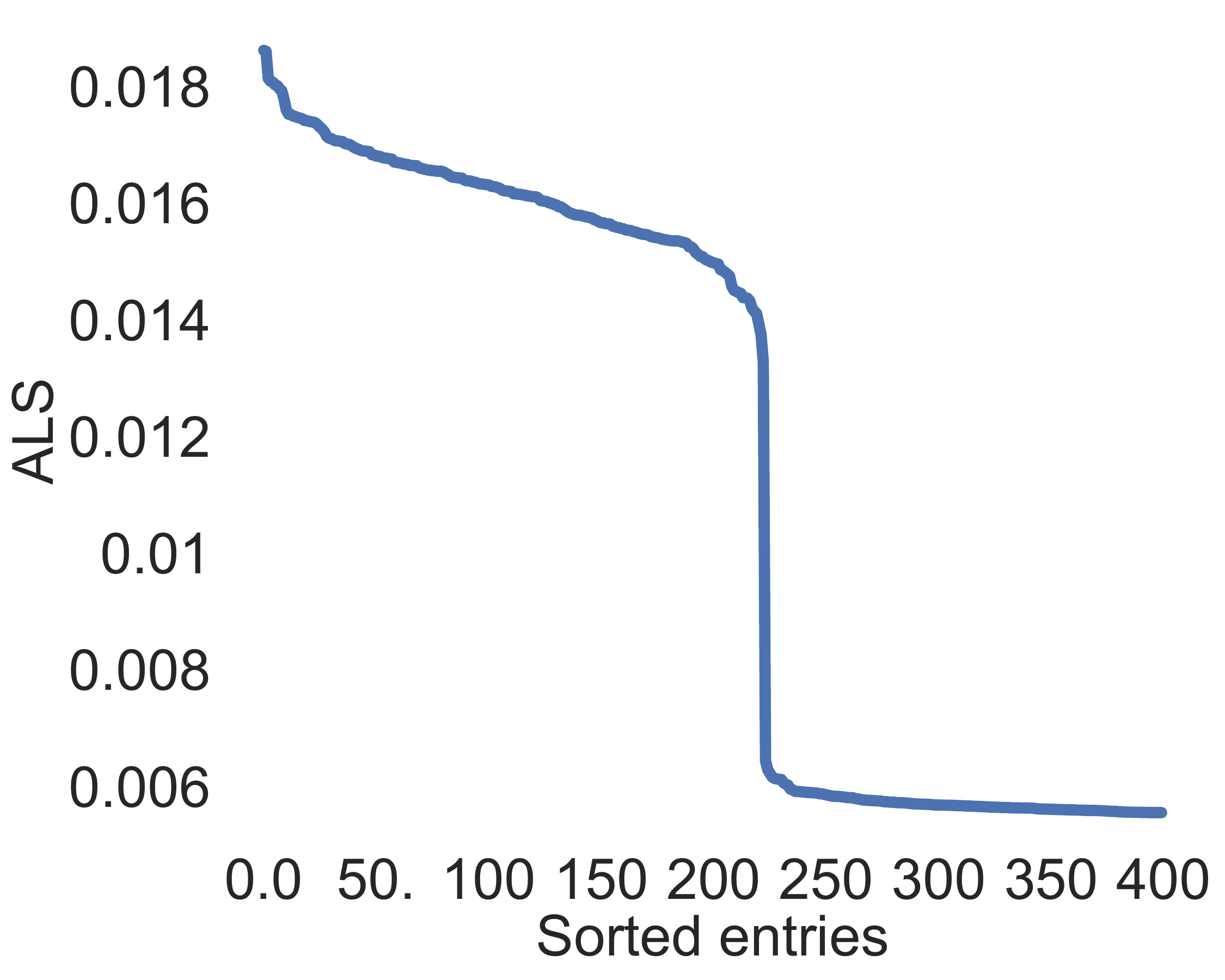}
\hfill
\includegraphics[width=0.288\textwidth]{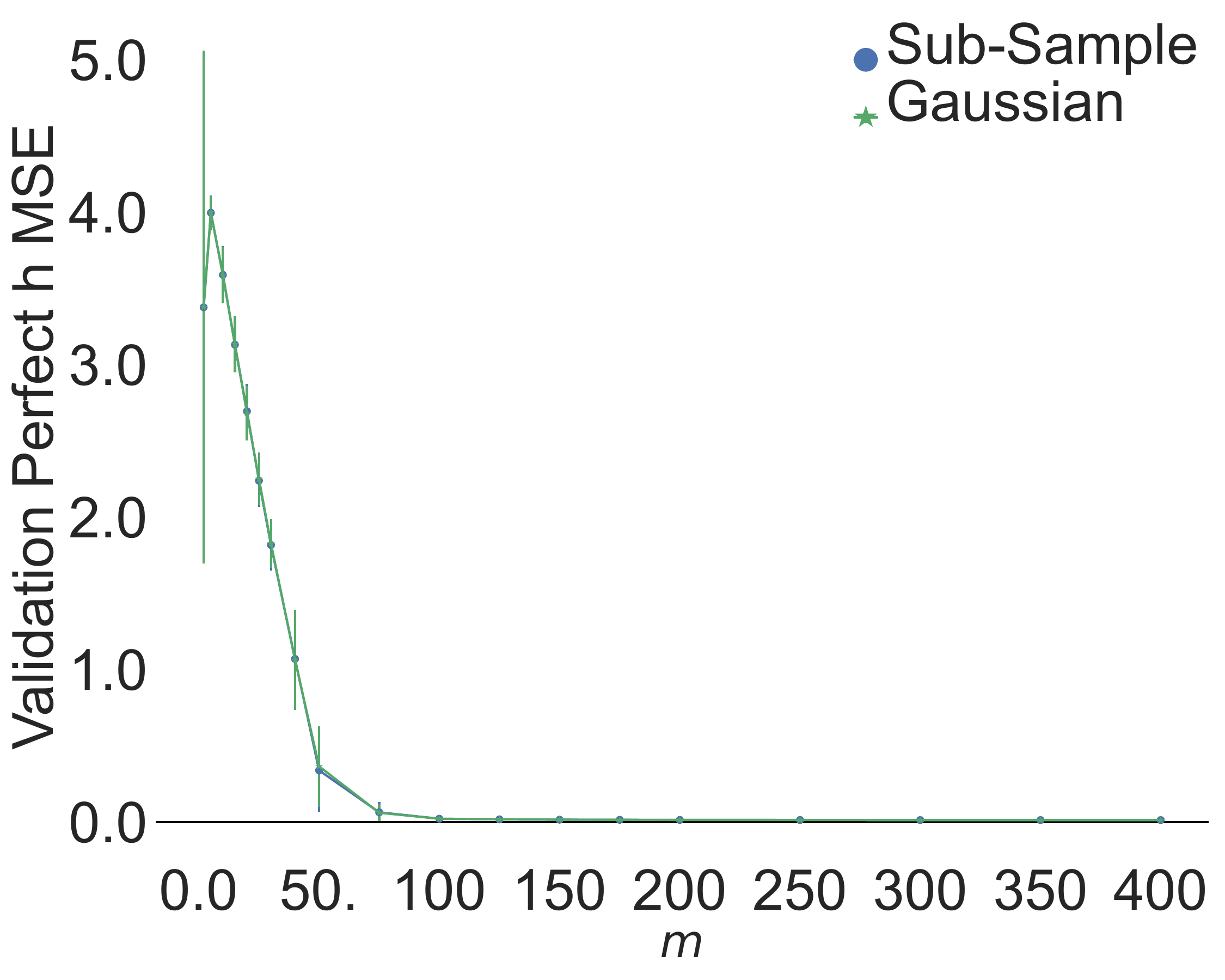}
\hfill
\includegraphics[width=0.288\textwidth]{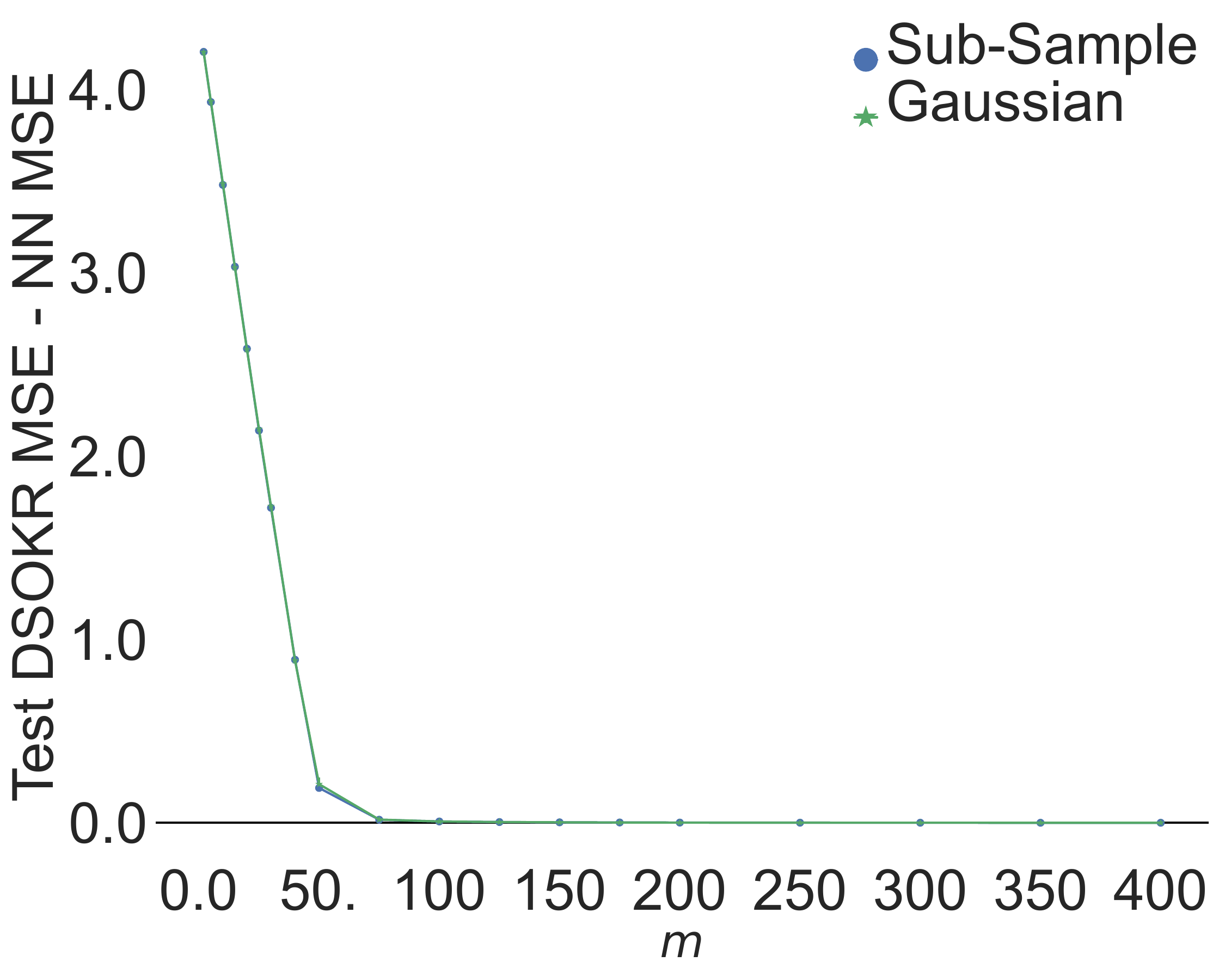}
\caption{Sorted 400 highest ALS (left), validation MSE of \textit{Perfect h} w.r.t. $\my$ (center) and the difference between test MSE of DSOKR and NN w.r.t. $\my$ (right).}
\label{fig:toy_figures}
\end{figure}
 
\paragraph{\bf Dataset.} We generate a synthetic dataset of least-squares regression, using then a linear output kernel, with $n=50,000$ training data points, $\bmX = \reals^{2,000}$, $\bmY = \reals^{1,000}$, and $\Hy = \bmY = \reals^{1,000}$.
The goal is to build this dataset such that the outputs lie in a subspace of $\bmY$ of dimension $d = 50 < 1,000$.
Hence, given $d$ randomly drawn orthonormal vectors $(u_j)_{j=1}^d$, for all $1 \leq i \leq n$, the outputs are such that $y_i = \sum_{j=1}^d \alpha(x_i)_j u_j + \varepsilon_i$, where $\alpha$ is a function of the inputs and $\varepsilon_i \sim \bmN(0, \sigma^2 I_{1,000})$ are i.i.d. with $\sigma^2 = 0.01$.
We generate i.i.d. normal distributed inputs $x_i \sim \bmN(0, C)$, where $(\sigma_j(C) = j^{-1/2})_{j=1}^{2,000}$ and its eigenvectors are randomly drawn.
Finally, we draw $H \in \reals^{d \times 2,000}$ with i.i.d. coefficients from the standard normal distribution, and the outputs are given for $1 \leq i \leq n$ by
\looseness-1
\begin{equation}
    y_i = U H x_i + \varepsilon_i\,,
\end{equation}
where $U = (u_1, \ldots, u_d) \in \reals^{1,000 \times d}$.
We generate validation and test sets of $n_\textnormal{val}=5,000$ and $n_\textnormal{te}=10,000$ points in the same way.

\paragraph{\bf Experimental settings.} We first compute the ALS as described above.
%
%
We take as regularisation penalty $\lambda = 10^{-4}$, sampling parameter $n_S = \sqrt{n}$ and probability vector $(p_i = 1/n)_{i=1}^n$ (uniform sampling).
Then, we perform the sketching size selection strategy \textit{Perfect h}.
%
%
Note that using a linear output kernel, $\psiy : y \in \reals^{1,000} \mapsto y$, then $\tilde{e}_i = (1/\sqrt{\sigma_i(\sketchGramY)}) \tilde{\mathbf{v}}_{i}^\top \sketchy Y$, where $Y = (y_1, \ldots, y_n)^\top \in \reals^{n \times 1,000}$, and
\looseness-1
\begin{equation}
    h_{\hat{\theta}}(x) = Y^\top \sketchy^\top \widetilde{V}_{\ranksketchGramY} \widetilde{D}_{\ranksketchGramY}^{-1 / 2} g_{\hat{W}}(x)\,.
\end{equation}
Finally, we perform our DSOKR model whose neural network $g_W$ is a Single-Layer Perceptron, i.e. with no hidden layer, and compare it with an SLP whose output size is $1,000$, and trained with a Mean Squared Error loss, that we call ``NN''.
%
We select the optimal number of epochs thanks to the validation set and evaluate the performance via the MSE.
We use the ADAM \citep{kingma2014adam} optimizer.
For the \textit{Perfect h} and DSOKR models and any sketching size $\my \in [2, 400]$, we average the results over five replicates of the models.
We use uniform sub-sampling without replacement and Gaussian sketching distributions.

\paragraph{\bf Experimental results.} \cref{fig:toy_figures} (left) presents the sorted 400 highest leverage scores.
This gives a rough estimate of the optimal sketching size since the leverage scores converge to a minimal value starting from 200 approximately, which is an upper bound of the true basis dimension $d = 50$.
\cref{fig:toy_figures} (center) shows that \textit{Perfect h} is a relevant strategy to fine-tune $\my$ since the obtained optimal value is $\my = 75$, which is very close to $d = 50$.
This small difference comes from the added noise $\varepsilon_i$.
Moreover, this value corresponds to the optimal value based on the DSOKR test MSE.
In fact, \cref{fig:toy_figures} (right) presents the performance DSOKR for many $\my$ values compared with NN.
DSOKR performance converges to the NN's performance for $\my = 75$ as well.
%
%
Hence, we show that DSOKR attains optimal performance if its sketching size is set as the dimension of the output marginal distribution's range, which can be estimated thanks to the ALS and the \textit{Perfect h} strategies.
There is no difference between sub-sample and Gaussian sketching since the dataset is rather simple.
Moreover, note that the neural network of the DSOKR model for $\my = 75$ contains $150,075$ parameters, whereas the NN model contains $2,001,000$ parameters.
Then, our sketched basis strategy, even in the context of multi-output regression, allows to reduce the size of the last layer, simplifying the regression problem and reducing the number of weights to learn.

\subsection{SMILES to Molecule: SMI2Mol}
\label{subsec:s2m}

\paragraph{\bf Dataset.} We use the QM9 molecule dataset \citep{ruddigkeit_enumeration_2012, ramakrishnan_quantum_2014}, containing around 130,000 small organic molecules. 
These molecules have been processed using RDKit\footnote{RDKit: Open-source cheminformatics. \url{https://www.rdkit.org}}, with aromatic rings converted to their Kekule form and hydrogen atoms removed. We also remove molecules containing only one atom. Each molecule contains up to 9 atoms of Carbon, Nitrogen, Oxygen, or Fluorine, along with three types of bonds: single, double, and triple. As input features, we use the Simplified Molecular Input Line-Entry System (SMILES), which are strings describing their chemical structure. We refer to the resulting dataset as \textbf{SMI2Mol}.

\paragraph{\bf Experimental set-up.} Using all SMILES-Molecule pairs, we build five splits using different seeds. Each split has 131,382 training samples, 500 validation samples, and 2,000 test samples.
In DSOKR, $g_{W}$ is a Transformer \citep{vaswani_attention_2017}. 
%
The SMILES strings are tokenized into character sequences as inputs for the Transformer encoder. To define the loss on output molecules, we cross-validate several graph kernels, including the Weisfeiler-Lehman subtree kernel (WL-VH) \citep{JMLR:v12:shervashidze11a}, the neighborhood subgraph pairwise distance kernel (NSPD) \citep{costa_fast_2010}, and the core Weisfeiler-Lehman subtree kernel (CORE-WL) \citep{nikolentzos_degeneracy_2018}. 
We use the implementation of the graph kernels provided by the Python library GraKel \citep{grakel}. 
We employ SubSample sketching for the output kernel.
The sketching size $m$ is fixed using our proposed \textit{Perfect h} strategy. Our method is benchmarked against SISOKR \citep{elahmad2024sketch}, NNBary-FGW \citep{brogat-motte_learning_2022}, and ILE-FGW \citep{brogat-motte_learning_2022}. For ILE-FGW and SISOKR, we additionally use SubSample sketching \citep{rudi2015less} for input kernel approximation. To ensure a fair comparison, both SISOKR and ILE-FGW adopt the 3-gram kernel for the input strings, whereas NNBary-FGW and DSOKR use a Transformer encoder. The performance is evaluated using Graph Edit Distance (GED), implemented by the NetworkX package \citep{hagberg_exploring_2008}.\looseness-1

\begin{figure}[!t]
\centering
%
\includegraphics[width=0.32\textwidth]{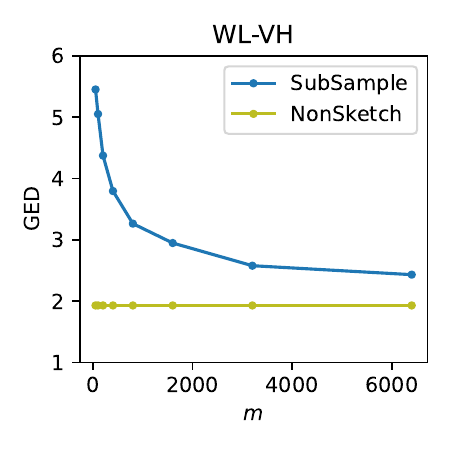}
\hfill
\includegraphics[width=0.32\textwidth]{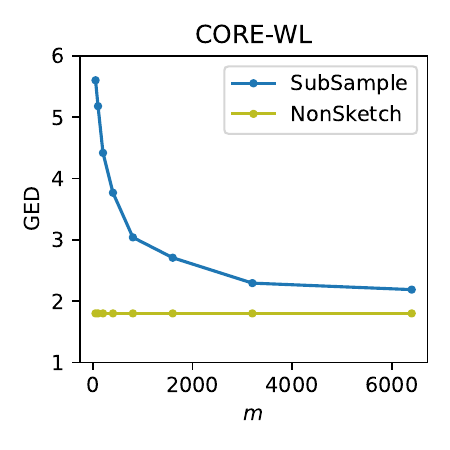}
\hfill
\includegraphics[width=0.32\textwidth]{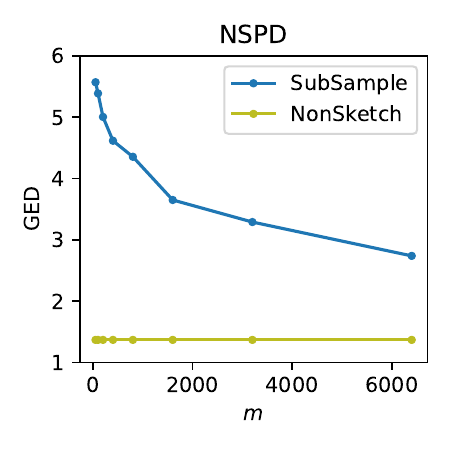}
\caption{The GED w/ edge feature w.r.t. the sketching size $\my$ for \textit{Perfect h} for three graph kernels on SMI2Mol ($\my > 6400$ is too costly computationally).}
\label{fig:hperfect_smi2mol}
\end{figure}


\begin{table}[t!]
\caption{Edit distance of different methods on SMI2Mol test set}
\begin{center}
\begin{tabular}{lcc}
\toprule
 & GED w/o edge feature $\downarrow$ &  GED w/ edge feature $\downarrow$ \\
\midrule
SISOKR & $3.330 \pm 0.080$ & $4.192 \pm 0.109$ \\
NNBary-FGW & $5.115 \pm 0.129$ & - \\
Sketched ILE-FGW & $2.998 \pm 0.253$  & - \\
\midrule
DSOKR & $\textbf{1.951} \pm \textbf{0.074}$& $\textbf{2.960} \pm \textbf{0.079}$\\
\bottomrule
\end{tabular}
\end{center}
\label{tab:expr_s2m_5s}
\end{table}

\paragraph{\bf Experimental results.} 
Figure \ref{fig:hperfect_smi2mol} displays the GED obtained by \textit{Perfect h} concerning various graph kernels. Based on this visualization, we have set the sketching sizes of WL-VH, CORE-WL, and NPSD to 3200, 3200, and 6400 respectively. Table \ref{tab:expr_s2m_5s} showcases the performance of various methods of SGP. Notably, DSOKR outperforms all baseline methods. It is evident that while graph kernels and the fused Gromov-Wasserstein (FGW) distance induce a meaningful feature space, the capabilities of SISOKR and ILE-FGW are constrained by the input kernels, thus highlighting the relevance of our proposed method. For further insight, a comparison of some prediction examples is provided in \cref{fig:smi2mol_examples} and \cref{apx:add_xps_s2m}.
\looseness-1

\begin{figure}[!t]
        \centering
        \begin{subfigure}[b]{0.19\textwidth}
            \centering
            \includegraphics[width=\textwidth]{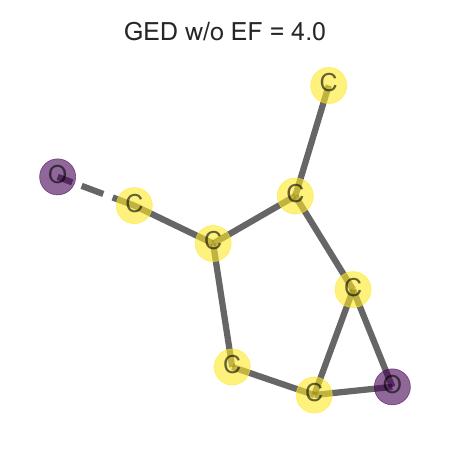}
        \end{subfigure}
        \hfill
        \begin{subfigure}[b]{0.19\textwidth}  
            \centering 
            \includegraphics[width=\textwidth]{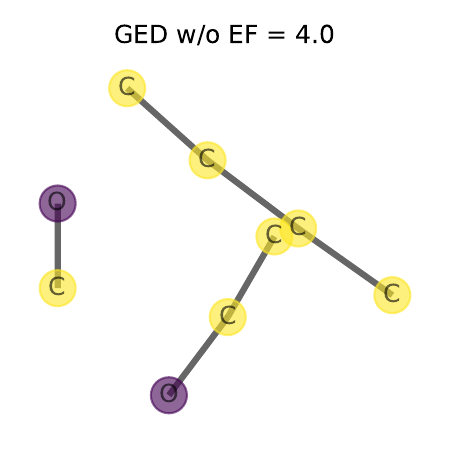} 
        \end{subfigure}
        \hfill
        \begin{subfigure}[b]{0.19\textwidth}  
            \centering 
            \includegraphics[width=\textwidth]{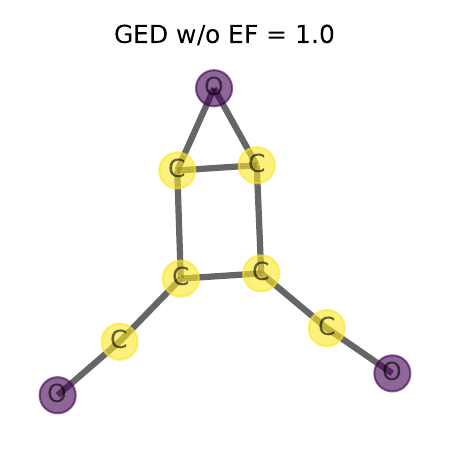} 
        \end{subfigure}
        \hfill
        \begin{subfigure}[b]{0.19\textwidth}  
            \centering 
            \includegraphics[width=\textwidth]{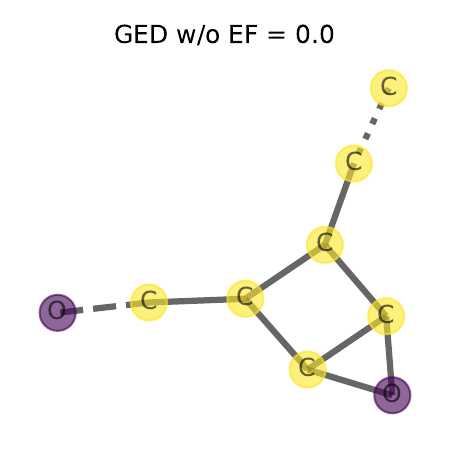} 
        \end{subfigure}
        \hfill
        \begin{subfigure}[b]{0.19\textwidth}  
            \centering 
            \includegraphics[width=\textwidth]{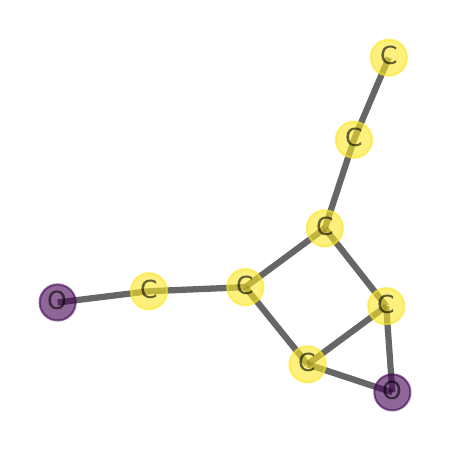} 
        \end{subfigure}
        \vskip\baselineskip
        \begin{subfigure}[b]{0.19\textwidth}
            \centering
            \includegraphics[width=\textwidth]{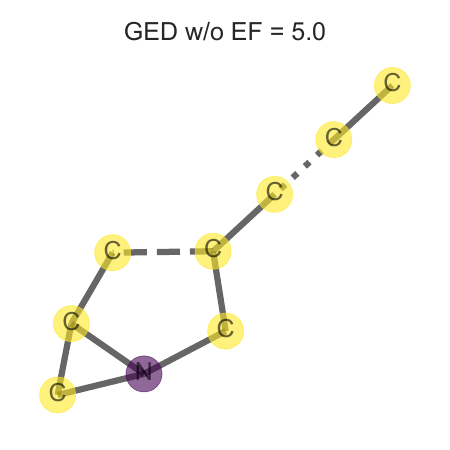}
            \caption{SISOKR}
        \end{subfigure}
        \hfill
        \begin{subfigure}[b]{0.19\textwidth}  
            \centering 
            \includegraphics[width=\textwidth]{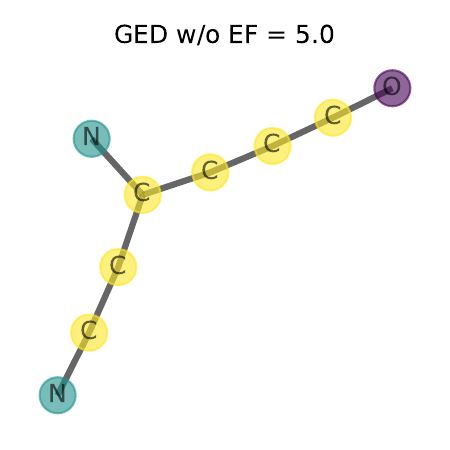} 
            \caption{NNBary}
        \end{subfigure}
        \hfill
        \begin{subfigure}[b]{0.19\textwidth}  
            \centering 
            \includegraphics[width=\textwidth]{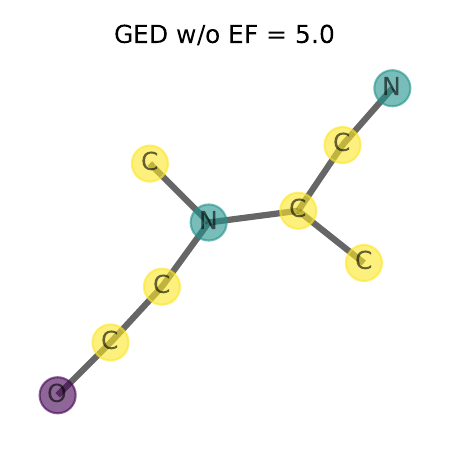} 
            \caption{ILE}
        \end{subfigure}
        \hfill
        \begin{subfigure}[b]{0.19\textwidth}  
            \centering 
            \includegraphics[width=\textwidth]{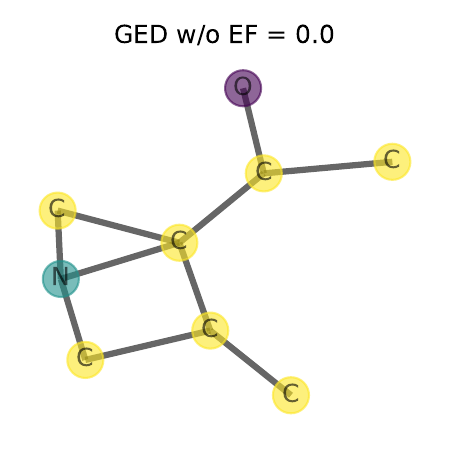}
            \caption{DSOKR}
        \end{subfigure}
        \hfill
        \begin{subfigure}[b]{0.19\textwidth}  
            \centering 
            \includegraphics[width=\textwidth]{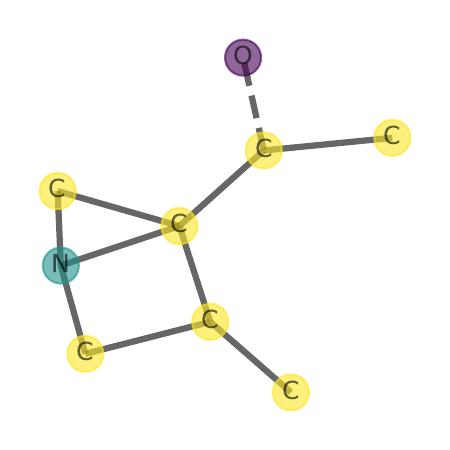}
            \caption{True target}
        \end{subfigure}
        \caption{\small Predicted molecules on the SMI2Mol dataset.} 
        \label{fig:smi2mol_examples}
\end{figure}

\subsection{Text to Molecule: ChEBI-20}
\label{subsec:t2m}

\begin{figure}[!t]
\centering
\includegraphics[width=0.4\textwidth]{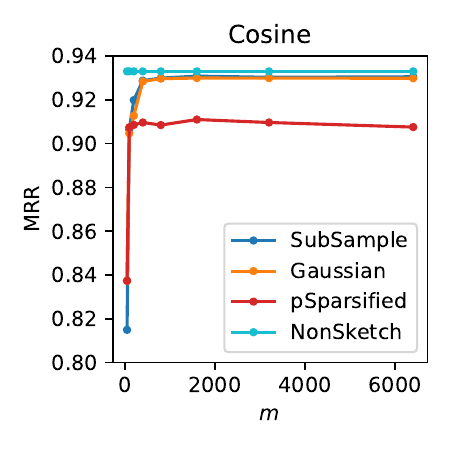}\qquad
\includegraphics[width=0.4\textwidth]{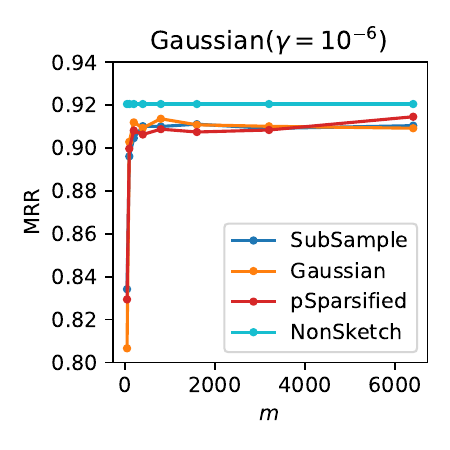}
\caption{The MRR scores on ChEBI-20 validation set w.r.t. $\my$ for \textit{Perfect h} when the output kernel is Cosine or Gaussian on the ChEBI-20 dataset.}
\label{fig:chebi_hperfect}
\end{figure}

\begin{table}[t!]
\caption{Performance of different methods on ChEBI-20 test set. All the methods based on NNs use SciBERT as input text encoder for fair comparison. The number in the ensemble setting indicates the number of single models used.}
\begin{center}
\begin{tabular}{lccccc }
\toprule
  & Hits@1 $\uparrow$ & Hits@10 $\uparrow$ & MRR $\uparrow$ \\
\midrule
SISOKR & 0.4\% & 2.8\% &  0.015 \\
SciBERT Regression & 16.8\% & 56.9\% & 0.298 \\
\hdashline
CMAM - MLP & 34.9\% & 84.2\% &  0.513 \\
CMAM - GCN   & 33.2\% & 82.5\%  & 0.495 \\
CMAM - Ensemble (MLP$\times 3$)  &39.8\%  & 87.6\%  & 0.562 \\
CMAM - Ensemble (GCN$\times 3$)  & 39.0\% &  87.0\% &  0.551 \\
CMAM - Ensemble (MLP$\times 3$ + GCN$\times 3$) & 44.2\% & \textbf{88.7}\%  & 0.597   \\
\midrule
DSOKR - SubSample Sketch & 48.2\% & 87.4\% &   0.624   \\
DSOKR - Gaussian Sketch & 49.0\% & 87.5\%  &   0.630 \\
DSOKR - Ensemble (SubSample$\times 3$) & \textbf{51.0}\% & 88.2\% & \textbf{0.642}  \\
DSOKR - Ensemble (Gaussian$\times 3$) & 50.5\% & 87.9\%  & \textbf{0.642} \\
DSOKR - Ensemble (SubSample$\times 3$ + Gaussian$\times 3$) & 50.0\% & 88.3\% & 0.640 \\
\bottomrule
\end{tabular}
\end{center}
\label{tab:expr_t2m}
\end{table}

\paragraph{\bf Dataset.}
The ChEBI-20 \citep{edwards_text2mol_2021} dataset contains 33,010 pairs of compounds and descriptions.
The compounds come from PubChem \citep{kim_pubchem_2016, kim_pubchem_2019}, and their descriptions (more than 20 words) from the Chemical Entities of Biological Interest (ChEBI) database \citep{hastings_chebi_2016}.
The dataset is divided as follows: 80\% for training, 10\% for validation, and 10\% for testing.
The candidate set contains all compounds.
The mean and median number of atoms per molecule is 32 and 25 respectively, and the mean and median number of words per description is 55 and 51 respectively.

\paragraph{\bf Experimental set-up.} For our method DSOKR, we use SciBERT \citep{beltagy_scibert_2019} with an additional linear layer to parameterize $g_{W}$. The maximum length of the input tokens is set to 256. Mol2vec \citep{jaeger_mol2vec_2018} is used as the output molecule representation, which is a vector of dimension 300. Based on the Mol2vec representation, we conduct cross-validation using the following kernels: Cosine kernel and Gaussian kernel with gamma chosen from $\{10^{-9}, 10^{-6}, 10^{-3}, 1\}$, along with the following three sketches: sub-sampling \citep{rudi2015less}, Gaussian \citep{Yang2017}, and $p$-sparsified \citep{elahmad2023fast}. The sketching size for all combinations of the output kernels and sketches is determined using the \textit{Perfect h} strategy. As for the baselines, we consider SciBERT Regression, Cross-Modal Attention Model (CMAM) \citep{edwards_text2mol_2021}, and SISOKR. In the case of SciBERT Regression, we address the regression problem using Mean Squared Error loss, where the output space is the embedding space of Mol2vec, within a function space parameterized by SciBERT. CMAM aims to enhance the cosine similarity between the text embedding and the corresponding molecule in true pairs by employing a contrastive loss function. Specifically, the former is derived from SciBERT, while the latter is generated using either a multi-layer perceptron (MLP) or a graph convolutional network (GCN) atop the Mol2vec representation. We reproduce the results of CMAM with the codes\footnote{\url{https://github.com/cnedwards/text2mol}} released by \citet{edwards_text2mol_2021}. In SISOKR, we use SciBERT embeddings as input features, leveraging the cosine kernel atop them. We maintain the identical output kernel sketching setup as in DSOKR. For all methods, we train the model using the best hyper-parameters with three random seeds and report the one with the best validation performance. The performance is evaluated with mean reciprocal rank (MRR), Hits@1 and Hits@10.
We could not benchmark AMAN \citep{zhao_adversarial_2024}, as no implementation is publicly available.

\paragraph{\bf Ensemble.} In \citet{edwards_text2mol_2021}, the authors propose an ensemble strategy to enhance the results by aggregating the ranks obtained by different training of their models.
If for each $1 \leq t \leq T$, $R_t$ denotes the ranking returned by the model $t$, the new score is computed as follows
\begin{equation}
    s(y_i) = \sum_{t = 1}^T w_t R_t(y_i) \quad s.t. \quad \sum_{i=1}^T \omega_t = 1
\end{equation}
for each $y_i$ in the candidate set.
In our case, the computation of DSOKR's last layer $g_{E}$ depends on a draw of the sketching matrix $\sketchy$, which means that DSOKR is particularly well-suited to the aggregation via multiple draws of the sketching matrix $\sketchy_t$ and the training of the corresponding neural networks $g_{W_t}$.
%
Hence, we explore two more ways of aggregating multiple DSOKR models, by averaging or maximizing these models' scores, i.e. for any input $x$ and candidate $y$,
\begin{equation}
    s(x, y) = \sum_{t = 1}^T \omega_t ~ g_{\hat{W}_t}(x)^\top \sketchpsiy_t(y) \quad \text{or} \quad s(x, y) = \argmax_{1 \leq t \leq T} ~ g_{\hat{W}_t}(x)^\top \sketchpsiy_t(y)\,.
\end{equation}
We explore all three ensemble methods for DSOKR models and subsequently select the optimal one based on its validation performance.

\paragraph{\bf Experimental results.} 
Figure \ref{fig:chebi_hperfect} illustrates the validation MRR scores with \textit{Perfect h}, for many $m$ values, and either Cosine or Gaussian output kernels.
%
It is evident that for both the Cosine kernel and Gaussian kernel (with $\gamma = 10^{-6}$) employing various sketching methods, the MRR score stabilizes as the sketching size exceeds 100, and that Cosine outperforms Gaussian.
This observation allows us to choose $m = 100$, smaller than the original Mol2vec dimension, which is 300. 
Table \ref{tab:expr_t2m} presents a comprehensive comparison of DSOKR with various baseline models. Firstly, comparing DSOKR with SISOKR reveals the critical importance of employing deep neural networks when dealing with complex structured inputs and DSOKR makes it possible in the case of functional output space. Secondly, the notable improvement over SciBERT Regression underscores the value of employing kernel sketching to derive more compact and better output features, thereby facilitating regression problem-solving. Lastly, DSOKR outperforms the sota CMAP for both single and ensemble models. See \cref{apx:add_xps_t2m} for more details.
\looseness-1

\section{Conclusion}
\label{sec:ccl}

We designed a new architecture of neural networks able to minimize kernel-induced losses for structured prediction and achieving sota performance on molecular identification.
An interesting avenue for future work is to derive excess risk for this estimator by combining deep learning theory and surrogate regression bounds.

\begin{ack}
Funded by the European Union.
Views and opinions expressed are however those of the author(s) only and do not necessarily reflect those of the European Union or European Commission.
Neither the European Union nor the granting authority can be held responsible for them.
This project has received funding from the European Union’s Horizon Europe research and innovation programme under grant agreement 101120237 (ELIAS), the Télécom Paris research chair on Data Science and Artificial Intelligence for Digitalized Industry and Services (DSAIDIS) and the PEPR-IA through the project FOUNDRY.
\end{ack}

%
%
%

\bibliographystyle{apalike}
\bibliography{references}

\clearpage

\appendix

\section{Technical proof}
\label{apx:proof}

We here provide the proof of \cref{prop:new-loss}.

\propnewloss*

\begin{proof}
    For any pair $(x, y) \in \bmX \times \bmY$, the loss function is given by
    \begin{align*}\label{eq:loss}
        \left\|h_\theta(x) - \psiy(y)\right\|_{\Hy}^2 &= \left\| \sum_{i=1}^p {g_{W}(x)}_j \tilde{e}_j - \psiy(y)\right\|_{\Hy}^2 \\
        &\hspace{-1cm}= \sum_{i, j=1}^p {g_{W}(x)}_i {g_{W}(x)}_j \langle \tilde{e}_i, \tilde{e}_j \rangle_{\Hy}  - 2 \sum_{j=1}^p {g_{W}(x)}_j \langle \tilde{e}_j, \psi(y)\rangle_{\Hy} + k(y,y) \\
        &\hspace{-1cm}= \left\|g_W(x)\right\|_2^2 - 2 {g_{W}(x)}^\top \sketchpsiy(y) + k(y,y) \,,
    \end{align*}
since $\tilde{E}$ is an orthonormal basis, and $\langle \tilde{e}_j, \psi(y)\rangle_{\Hy} = \tilde{e}_j(y) = \sketchpsiy(y)_j$ by the reproducing property.
Noting that
\begin{equation}
    \left\|g_W(x) - \sketchpsiy(y)\right\|_{2}^2 = \left\|g_W(x)\right\|_2^2 - 2 {g_{W}(x)}^\top \sketchpsiy(y) +  \left\|\sketchpsiy(y)\right\|_2^2\,,
\end{equation}
 and that both $k(y,y)$ and $\left\|\sketchpsiy(y)\right\|_2^2$ are independent of $W$ concludes the proof.
\end{proof}

\section{Graph Prediction via Output Kernel Regression}
\label{apx:graph_pred}

In this section, we present kernel examples to tackle graph prediction via Output Kernel Regression.

A graph $G$ is defined by its sets of vertices $V$ and edges $E$.
Besides, it may contain either node labels or attributes, or edge labels, attributes, or weights.
Before giving some examples of kernels dealing directly with graphs, we present examples of kernels dealing with fingerprints.

\paragraph{\bf Fingerprints.} Indeed, when manipulating molecules, either for molecular property prediction or molecule identification, many works use fingerprints to represent graphs \citep{RALAIVOLA20051093, Brouard_ismb2016, brouard2016input, tripp2023tanimoto}.
A fingerprint is a binary vector of length $d \geq 1$ and each entry of the fingerprint encodes the presence or absence of a substructure within the graph based on a dictionary.
Hence, when using fingerprints, the problem of graph prediction becomes a high-dimensional multi-label prediction problem.
A very popular kernel to handle fingerprints is the Tanimoto kernel \citep{tanimoto1958elementary}, which basically consists of an intercept over union measure between two fingerprints.

\paragraph{\bf Graph kernels.} In this work, we also manipulate raw graphs.
Many kernels exist to handle graphs, we present a few that we will use during the experiments.
For more details about these kernels and other graph kernel examples, see the documentation of the GraKel library \citep{grakel}.
\begin{definition}[Vertex Histogram kernel]
    Let $G = (V, E)$ and $G^\prime = (V^\prime, E^\prime)$ be two node-labeled graphs.
    Let $\bmL = \{1, \ldots, d\}$ be the set of labels, and $\ell : v \in V \mapsto \ell(v) \in \bmL$ be the function that assigns a label for each vertex.
    Then, the vertex label histogram of $G$ is a vector $f = (f_1, \ldots, f_d)^\top$, such that $f_i = | \{v \in V : \ell(v) = i \} |$ for each $i \in \bmL$.
    Let $f, f^\prime$ be the vertex label histograms of $G, G^\prime$, respectively.
    The vertex histogram kernel is then defined as the linear kernel between $f$ and $f^\prime$, that is
    \begin{equation}
        \kernely(G, G^\prime) = f^\top f^\prime\,.
    \end{equation}
\end{definition}
The VH kernel needs node-labeled graphs and simply compares two graphs based on the number of nodes having each type of label.
Its computation is very fast.

\begin{definition}[Shortest-Path kernel \citep{BorgwardtSP}]
    Let $G = (V, E)$ and $G^\prime = (V^\prime, E^\prime)$ be two graphs, and $S = (V, E_S)$ and $S^\prime = (V^\prime, E_{S^\prime}^\prime)$ their corresponding shortest-path graphs, i.e. the graphs where we only keep the edges contained in the shortest path between every vertex, then $E_S \subseteq E$ and $E_{S^\prime}^\prime \subseteq E^\prime$.
    The shortest-path kernel is then defined on $G$ and $G^\prime$ as
    \begin{equation}
        \kernely(G, G^\prime) = \kernely(S, S^\prime) = \sum_{e \in E} \sum_{e^\prime \in E^\prime} k_{\operatorname{walk}}^{(1)}(e, e^\prime)\,,
    \end{equation}
    where $k_{\operatorname{walk}}^{(1)}(e, e^\prime)$ is a positive semidefinite kernel on edge walks of length $1$.
\end{definition}
The SP kernel can handle graphs either without node labels, with node labels, or with node attributes.
This information, as well as the shortest path lengths, are encoded into $k_{\operatorname{walk}}^{(1)}$ whose classical choices are Dirac kernels or, more rarely, Brownian bridge kernels.
The computation of the SP kernel is very expensive since it takes $\bmO(n_V)$ time, where $n_V$ denotes the number of nodes.

We present the Neighborhood Subgraph Pairwise Distance kernel \citep{costa_fast_2010}.
This kernel extracts pairs of subgraphs from each graph and then compares these pairs.
\begin{definition}[Neighborhood Subgraph Pairwise Distance kernel \citep{costa_fast_2010}]
    Let $G = (V, E)$ and $G^\prime = (V^\prime, E^\prime)$ be two node-labeled and egde-labeled graphs.
    For $u, v  \in V$, $D(u, v)$ denotes the distance between $u$ and $v$, i.e. the length of the shortest path between them, for $r \geq 1$, $\{ u \in V : D(u, v) \leq r \}$ denotes the neighborhood of radius $r$ of a vertex $v$, i.e. the set of vertices at a distance less than or equal to $r$ from $v$, for a subset of vertices $S \subseteq V$, $E(S)$ denotes the set of edges that have both end-points in $S$, and we can define the subgraph with vertex set $S$ and edge set $E(S)$.
    $N_r^v$ denotes the subgraph induced by $\{ u \in V : D(u, v) \leq r \}$.
    Let also $R_{r, d}(A_v, B_u, G)$ be a relation between two rooted graphs $A_v$, $B_u$ and a graph $G = (V, E)$ that is true if and only if both $A_v$ and $B_u$ are in $\{N_r^v : v \in V\}$, where we require $A_v, B_u$ to be isomorphic to some $N_r^v$ to verify the set inclusion, and that $D(u, v) = d$.
    We denote with $R^{-1}(G)$ the inverse relation that yields all the pairs of rooted graphs $A_v$, $B_u$ satisfying the above constraints.
    The neighborhood subgraph pairwise distance kernel is then based on the following kernel
    \begin{equation}
        k_{r, d}(G, G^\prime) = \sum_{A_v, B_u \in R_{r, d}^{-1}(G)} \sum_{A_{v^\prime}^\prime, B_{u^\prime}^\prime \in R_{r, d}^{-1}(G^\prime)} \delta(A_v, A_{v^\prime}^\prime) \delta(B_u, B_{u^\prime}^\prime)\,,
    \end{equation}
    where $\delta$ is $1$ if its input subgraphs are isomorphic, and $0$ otherwise.
    This counts the number of identical pairs of neighboring subgraphs of radius $r$ at distance $d$ between two graphs.
    The NSPD kernel is then defined on $G$ and $G^\prime$ as
    \begin{equation}
        \kernely(G, G^\prime) = \sum_{r = 0}^{r^*} \sum_{d = 0}^{d^*} \hat{k}_{r, d}(G, G^\prime)\,,
    \end{equation}
    where $\hat{k}_{r, d}$ is a normalized version of $k_{r, d}$, and $r^*$ and $d^*$ are hyper-parameters of the kernel.
\end{definition}
The NSPD takes into account the edge labels, which can be of particular interest when manipulating molecules.
For small values of $r^*$ and $d^*$, its complexity is in practice linear in the size of the graph.

We now introduce the Weisfeiler-Lehman framework, inspired by the Weisfeiler-Lehman test of graph isomorphism \citep{weisfeiler1968reduction}, that operates on top of existing graph kernels.
The Weisfeiler-Lehman algorithm replaces the label of each vertex with a multiset label consisting of the original label of the vertex and the sorted set of labels of its neighbors.
The resulting multiset is then compressed into a new, short label, and this procedure is repeated for $h$ iterations.
\begin{definition}[Weisfeiler-Lehman kernel \citep{JMLR:v12:shervashidze11a}]
    Let $G = (V, E)$ and $G^\prime = (V^\prime, E^\prime)$ be two node-labeled graphs, endowed with labeling functions $\ell = \ell_0$ and $\ell^\prime = \ell_0^\prime$, respectively.
    The Weisfeiler-Lehman graph of $G$ at height $i$ is a graph $G_i$ endowed with a labeling function $\ell_i$ which has emerged after $i$ iterations of the relabeling procedure described previously.
    Let $\kernely_{\operatorname{base}}$ be any kernel for graphs, called the base kernel.
    The Weisfeiler-Lehman kernel with $h$ iterations is then defined on $G$ and $G^\prime$ as
    \begin{equation}
        \kernely(G, G^\prime) = \kernely_{\operatorname{base}}(G_0, G_0^\prime) + \ldots + \kernely_{\operatorname{base}}(G_h, G_h^\prime)\,.
    \end{equation}
\end{definition}
A very popular choice is the Weisfeiler-Lehman subtree kernel, which corresponds to choosing
the VH kernel as the base kernel.
Its time complexity is $\bmO(h n_E)$, where $n_E$ denotes the number of edges, which is efficient.
We call it the WL-VH kernel.

We finally present the Core kernel framework that, similarly to the WL framework, operates on top of existing graph kernels.
It builds upon the notion of $k$-core decomposition, first introduced to study the cohesion of social networks \citep{SEIDMAN1983269}.
\begin{definition}[Core kernel \citep{nikolentzos_degeneracy_2018}]
    Let $G = (V, E)$ and $G^\prime = (V^\prime, E^\prime)$ be two graphs.
    Let $G_{\operatorname{sub}}(S, E(S))$ be the subgraph induced by the subset of vertices $S \subseteq V$ and the set of edges $E(S)$ that have both end-points in $S$.
    Let $d_{G_{\operatorname{sub}}}(v)$ be the degree of a vertex $v \in S$, i.e. the number of vertices that are adjacent to $v$ in $G_{\operatorname{sub}}$.
    The, $G_{\operatorname{sub}}$ is a $k$-core of $G$, denoted by $C_k$, if it is a maximal subgraph of $G$ in which all vertices have a degree at least $k$.
    Let $\kernely_{\operatorname{base}}$ be any kernel for graphs, called the base kernel.
    The core variant of this kernel is then defined on $G$ and $G^\prime$ as
    \begin{equation}
        \kernely(G, G^\prime) = \kernely_{\operatorname{base}}(C_0, C_0^\prime) + \ldots + \kernely_{\operatorname{base}}(C_{\delta_{\operatorname{min}}^*}, C_{\delta_{\operatorname{min}}^*}^\prime)\,,
    \end{equation}
    where $\delta_{\operatorname{min}}^*$ is the minimum of the degeneracies of the two graphs, and for all $1 \leq i \leq \delta_{\operatorname{min}}^*$, $C_i$ and $C_i^\prime$ are the $i$-core subgraphs of $G$ and $G^\prime$.
\end{definition}
The time complexity of computing the $k$-core decomposition of a graph is $\bmO(n_V + n_E)$.
Moreover, the complexity of computing the core variant of a kernel depends on its complexity, and in general, the complexity added by the core variant is not very high.

\section{Additional Experimental Details}
\label{apx:add_xps}

In this section, we report additional experimental details on both SMI2Mol and ChEBI-20 datasets.

\subsection{Additional Experimental Details for SMI2Mol}
\label{apx:add_xps_s2m}
For DSOKR, we optimize the parameters of neural networks using Adam with a learning rate of $10^{-3}$ over 50 epochs. We adopt early stopping based on the validation set's edit distance. The number of transformer layers is chosen from $\{3, 6\}$, the model dimension is selected from $\{256, 512\}$, the number of heads is set to $8$, the feed-forward network dimension is set to four times the model dimension, and the dropout probability is set to $0.2$.

More examples of predictions can be found in \cref{fig:smi2mol_examples_add}.

\begin{figure}[!t]
        \centering
        \begin{subfigure}[b]{0.19\textwidth}
            \centering
            \includegraphics[width=\textwidth]{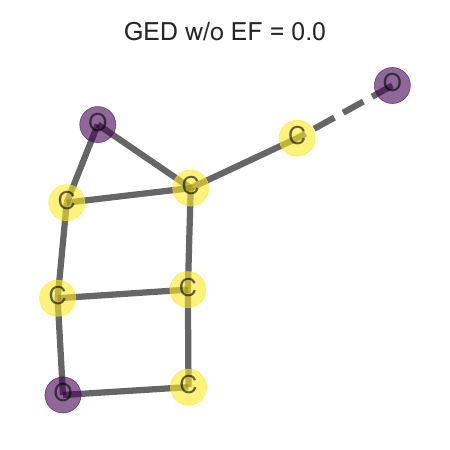}
        \end{subfigure}
        \hfill
        \begin{subfigure}[b]{0.19\textwidth}  
            \centering 
            \includegraphics[width=\textwidth]{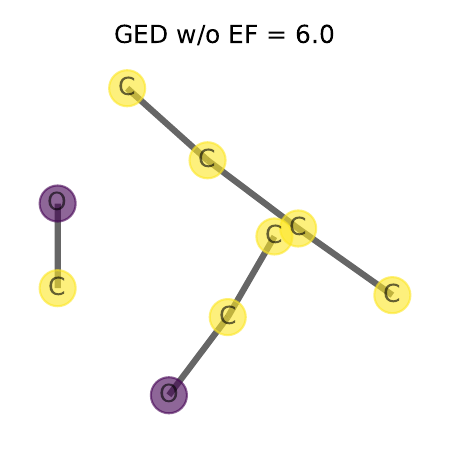} 
        \end{subfigure}
        \hfill
        \begin{subfigure}[b]{0.19\textwidth}  
            \centering 
            \includegraphics[width=\textwidth]{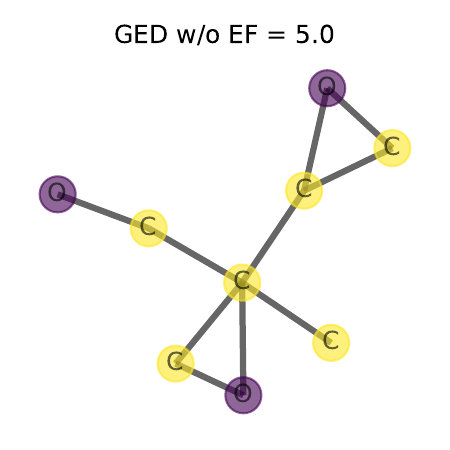} 
        \end{subfigure}
        \hfill
        \begin{subfigure}[b]{0.19\textwidth}  
            \centering 
            \includegraphics[width=\textwidth]{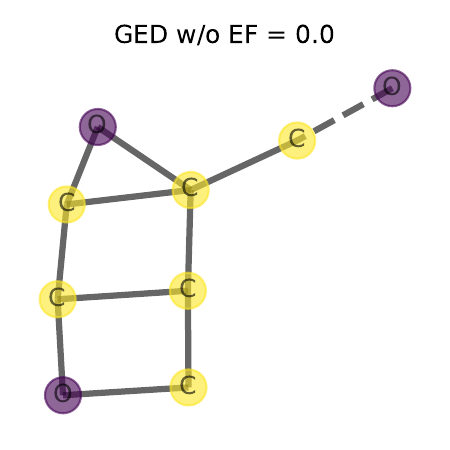} 
        \end{subfigure}
        \hfill
        \begin{subfigure}[b]{0.19\textwidth}  
            \centering 
            \includegraphics[width=\textwidth]{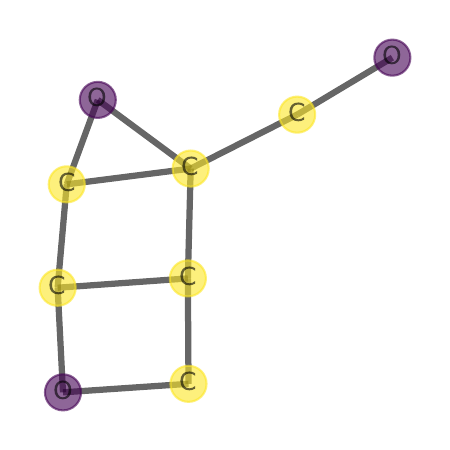} 
        \end{subfigure}
        \vskip\baselineskip
        \begin{subfigure}[b]{0.19\textwidth}
            \centering
            \includegraphics[width=\textwidth]{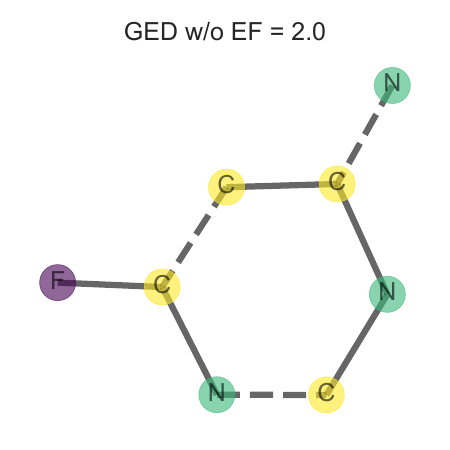}
        \end{subfigure}
        \hfill
        \begin{subfigure}[b]{0.19\textwidth}  
            \centering 
            \includegraphics[width=\textwidth]{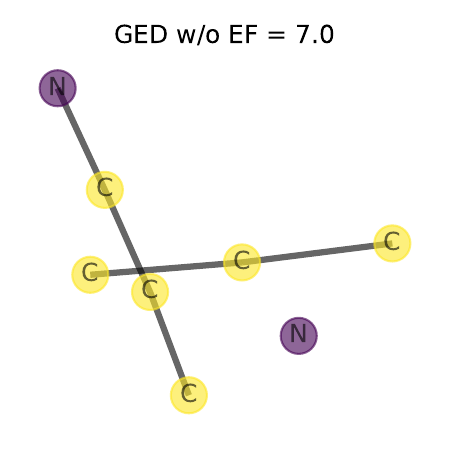} 
        \end{subfigure}
        \hfill
        \begin{subfigure}[b]{0.19\textwidth}  
            \centering 
            \includegraphics[width=\textwidth]{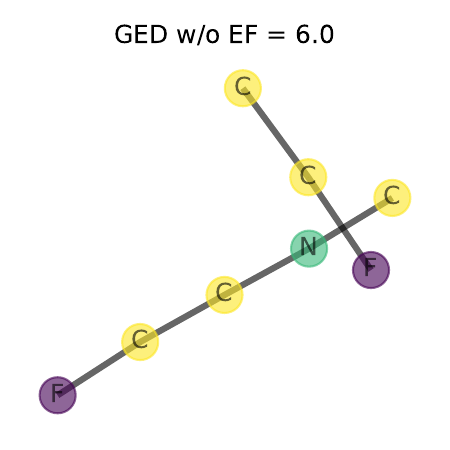} 
        \end{subfigure}
        \hfill
        \begin{subfigure}[b]{0.19\textwidth}  
            \centering 
            \includegraphics[width=\textwidth]{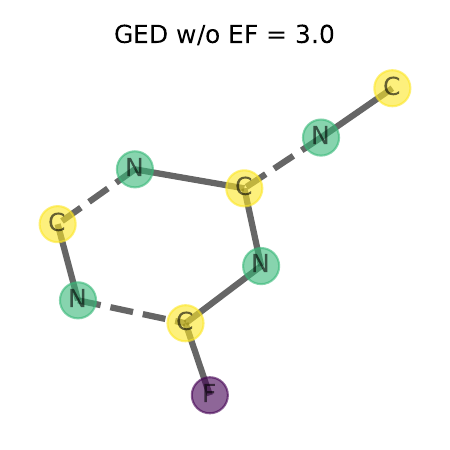} 
        \end{subfigure}
        \hfill
        \begin{subfigure}[b]{0.19\textwidth}  
            \centering 
            \includegraphics[width=\textwidth]{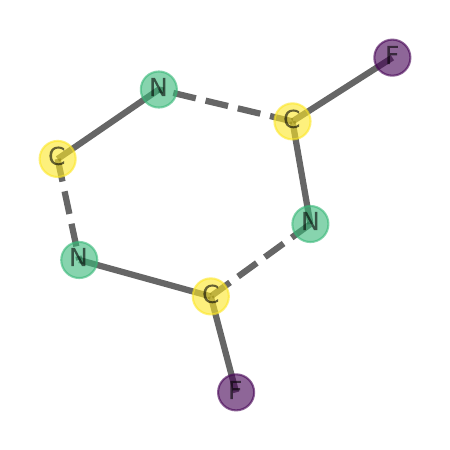} 
        \end{subfigure}
        \vskip\baselineskip
        \begin{subfigure}[b]{0.19\textwidth}
            \centering
            \includegraphics[width=\textwidth]{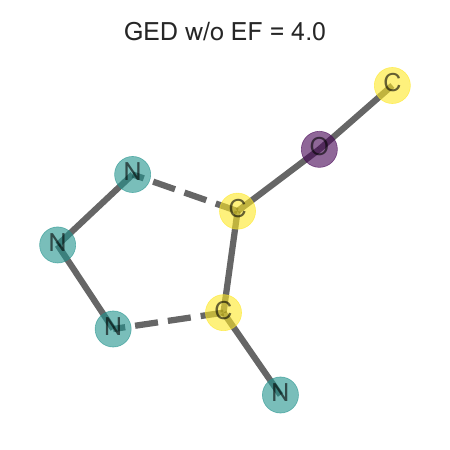}
        \end{subfigure}
        \hfill
        \begin{subfigure}[b]{0.19\textwidth}  
            \centering 
            \includegraphics[width=\textwidth]{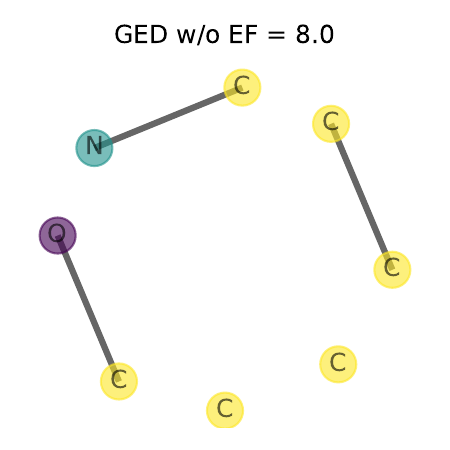} 
        \end{subfigure}
        \hfill
        \begin{subfigure}[b]{0.19\textwidth}  
            \centering 
            \includegraphics[width=\textwidth]{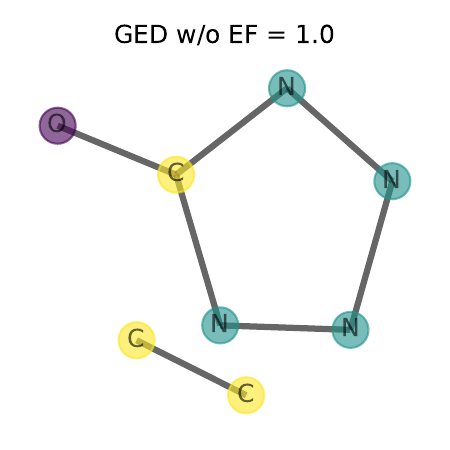} 
        \end{subfigure}
        \hfill
        \begin{subfigure}[b]{0.19\textwidth}  
            \centering 
            \includegraphics[width=\textwidth]{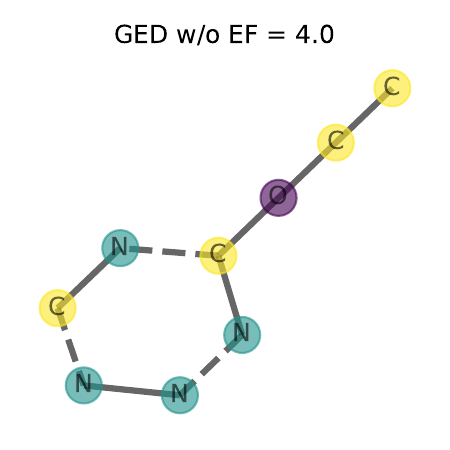} 
        \end{subfigure}
        \hfill
        \begin{subfigure}[b]{0.19\textwidth}  
            \centering 
            \includegraphics[width=\textwidth]{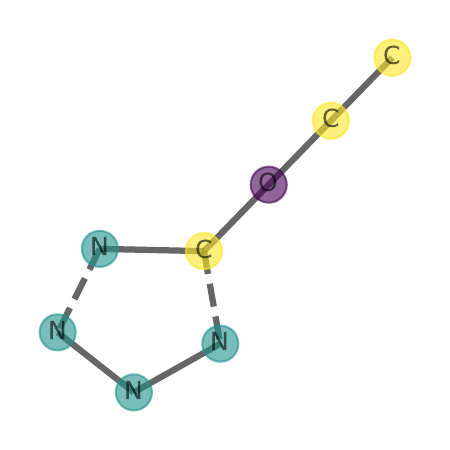} 
        \end{subfigure}
        \vskip\baselineskip
        \begin{subfigure}[b]{0.19\textwidth}
            \centering
            \includegraphics[width=\textwidth]{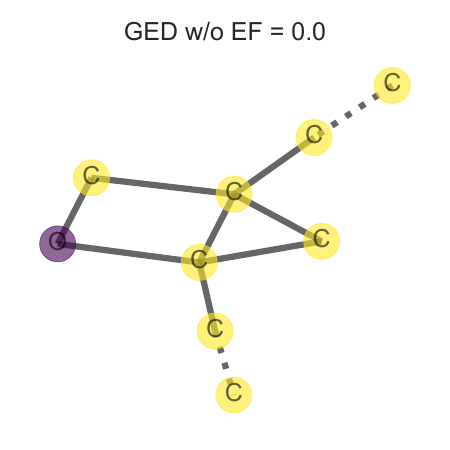}
        \end{subfigure}
        \hfill
        \begin{subfigure}[b]{0.19\textwidth}  
            \centering 
            \includegraphics[width=\textwidth]{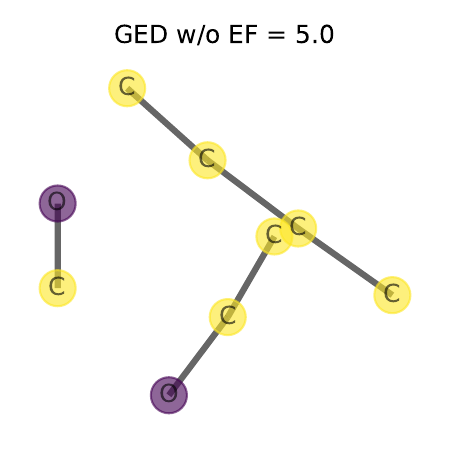} 
        \end{subfigure}
        \hfill
        \begin{subfigure}[b]{0.19\textwidth}  
            \centering 
            \includegraphics[width=\textwidth]{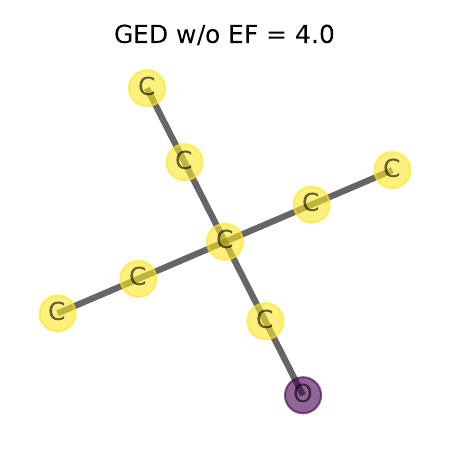} 
        \end{subfigure}
        \hfill
        \begin{subfigure}[b]{0.19\textwidth}  
            \centering 
            \includegraphics[width=\textwidth]{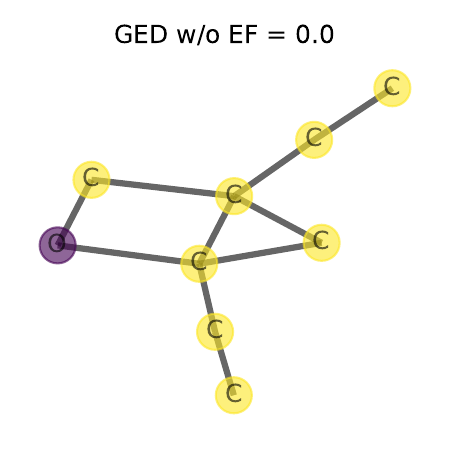} 
        \end{subfigure}
        \hfill
        \begin{subfigure}[b]{0.19\textwidth}  
            \centering 
            \includegraphics[width=\textwidth]{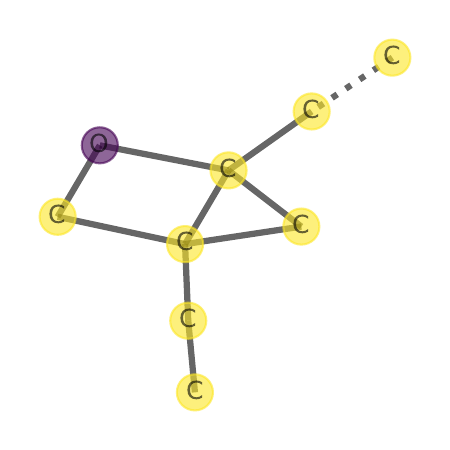} 
        \end{subfigure}
        \vskip\baselineskip
        \begin{subfigure}[b]{0.19\textwidth}
            \centering
            \includegraphics[width=\textwidth]{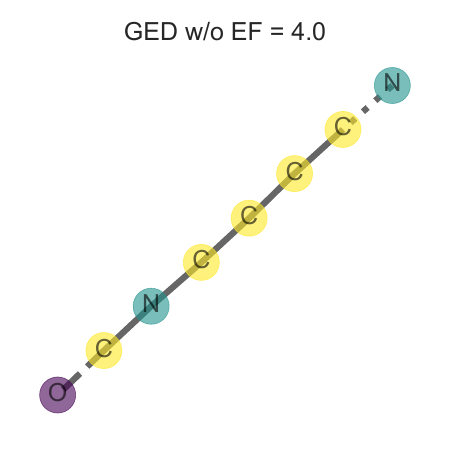}
        \end{subfigure}
        \hfill
        \begin{subfigure}[b]{0.19\textwidth}  
            \centering 
            \includegraphics[width=\textwidth]{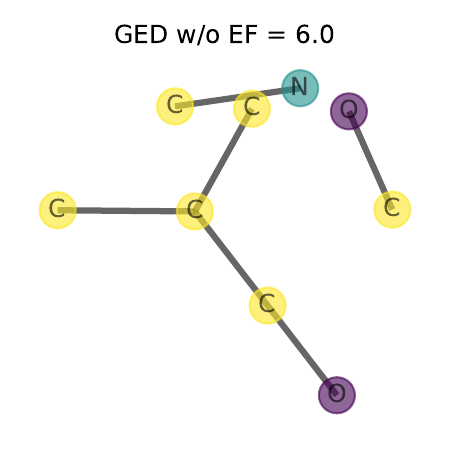} 
        \end{subfigure}
        \hfill
        \begin{subfigure}[b]{0.19\textwidth}  
            \centering 
            \includegraphics[width=\textwidth]{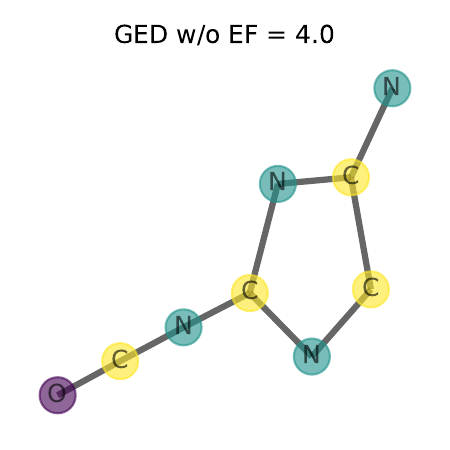} 
        \end{subfigure}
        \hfill
        \begin{subfigure}[b]{0.19\textwidth}  
            \centering 
            \includegraphics[width=\textwidth]{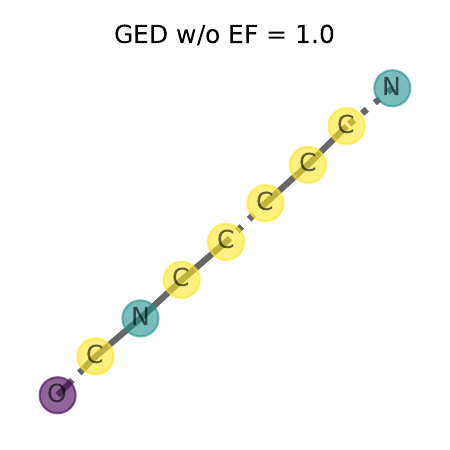} 
        \end{subfigure}
        \hfill
        \begin{subfigure}[b]{0.19\textwidth}  
            \centering 
            \includegraphics[width=\textwidth]{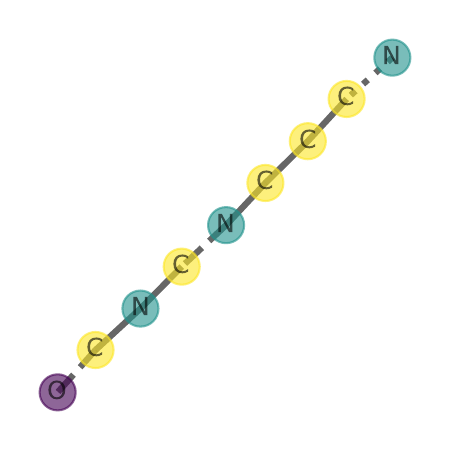} 
        \end{subfigure}
        \vskip\baselineskip
        \begin{subfigure}[b]{0.19\textwidth}
            \centering
            \includegraphics[width=\textwidth]{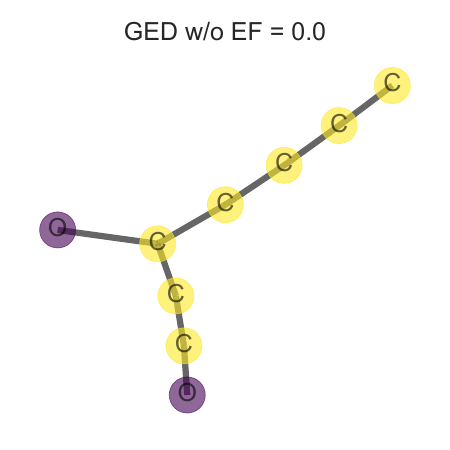}
            \caption{SISOKR}
        \end{subfigure}
        \hfill
        \begin{subfigure}[b]{0.19\textwidth}  
            \centering 
            \includegraphics[width=\textwidth]{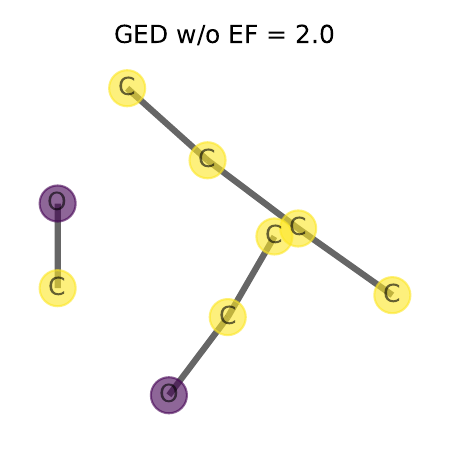} 
            \caption{NNBary}
        \end{subfigure}
        \hfill
        \begin{subfigure}[b]{0.19\textwidth}  
            \centering 
            \includegraphics[width=\textwidth]{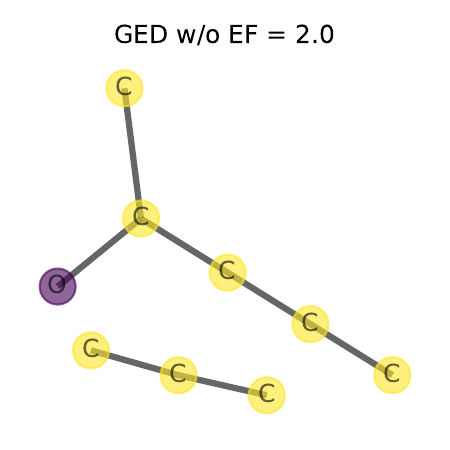} 
            \caption{ILE}
        \end{subfigure}
        \hfill
        \begin{subfigure}[b]{0.19\textwidth}  
            \centering 
            \includegraphics[width=\textwidth]{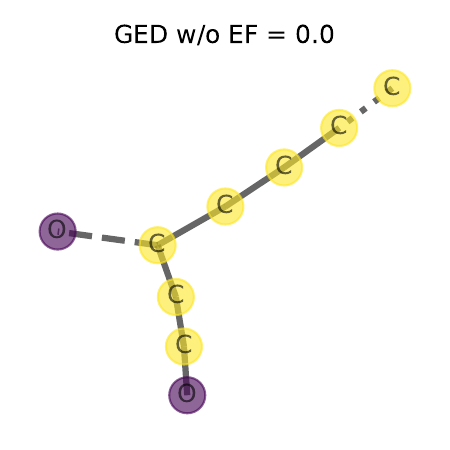} 
            \caption{DSOKR}
        \end{subfigure}
        \hfill
        \begin{subfigure}[b]{0.19\textwidth}  
            \centering 
            \includegraphics[width=\textwidth]{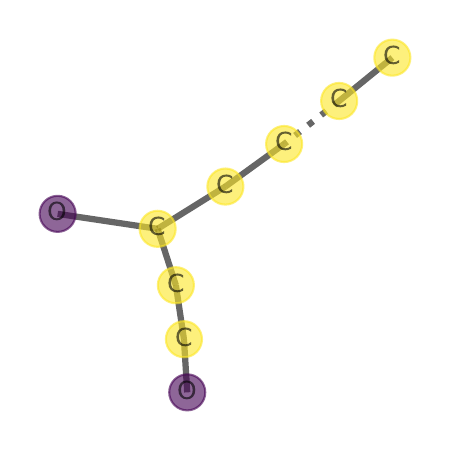} 
            \caption{True target}
        \end{subfigure}
        \caption{\small More predicted molecules on the SMI2Mol dataset.} 
        \label{fig:smi2mol_examples_add}
\end{figure}

\subsection{Additional Experimental Details for ChEBI-20}
\label{apx:add_xps_t2m}
For DSOKR, we conducted training on SciBERT for 50 epochs using the Adam optimizer with a learning rate of $3 \times 10^{-5}$. Additionally, we implemented a learning rate schedule that linearly decreases from the initial rate set by the optimizer to 0, following a warm-up period of 1000 steps where it linearly increases from 0 to the initial rate. We incorporated early stopping based on the MRR score on the validation set as well.

\cref{fig:chebi_hperfect_gamma} presents the validation MRR with respect to $\my$ obtained by \textit{Perfect h} with a Gaussian output kernel and additional values of $\gamma$.
The best $\gamma$ is clearly $10^{-6}$ since all sketching types attain the performance of the non-sketched \textit{Perfect h}.
\cref{tab:expr_t2m_complete} presents all the results gathered on ChEBI-20 with the additional Mean Rank metric.
DSOKR under-performs in terms of mean rank compared with CMAM while outperforming it in terms of hits@1, attaining around 50\%, and being equivalent to the ensemble CMAM methods in terms of hits@10, attaining around 88\%, which means that most of the time, the correct molecule is predicted in the top rankings and even at the top position half of the time, but in the 12\% left, the correct molecule falls to a high predicted rank.

\begin{figure}[!t]
\centering
\includegraphics[width=0.48\textwidth]{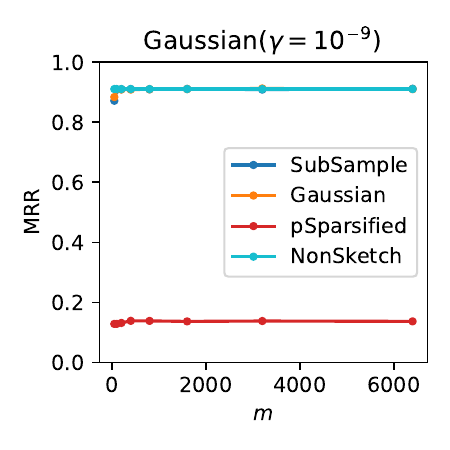}
\includegraphics[width=0.48\textwidth]{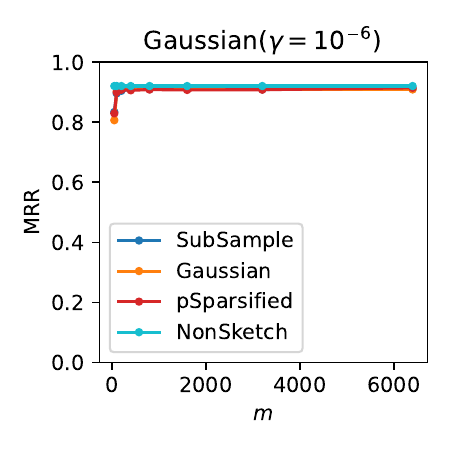}
\includegraphics[width=0.48\textwidth]{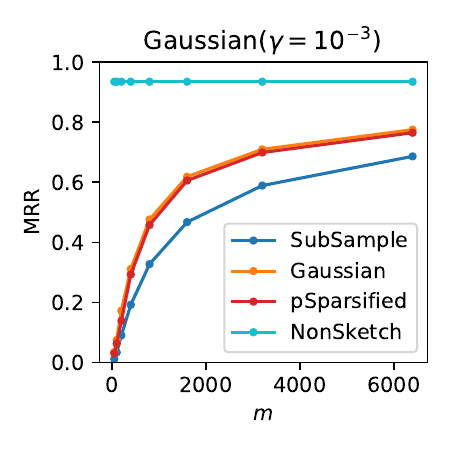}
\includegraphics[width=0.48\textwidth]{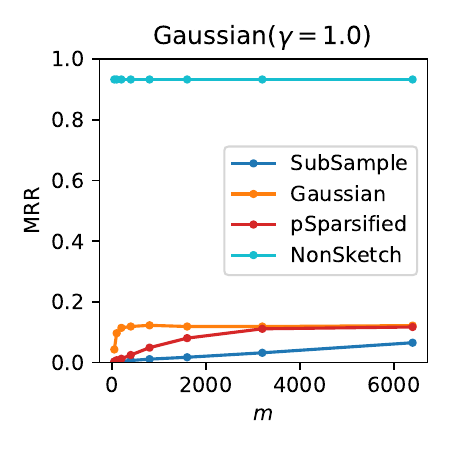}
\caption{The MRR scores on ChEBI-20 validation set with respect to the sketching size $m$ for \textit{Perfect h} when the output kernel is Gaussian with $\gamma \in \{10^{-9}, 10^{-6}, 10^{-3}, 1.0\}$.}
\label{fig:chebi_hperfect_gamma}
\end{figure}

\begin{table*}[t!]
\caption{Performance of different methods on ChEBI-20 test set. All the methods based on NNs use SciBERT as input text encoder for fair comparison.}
\begin{center}
\begin{tabular}{lccccc }
\toprule
 &  Mean Rank $\downarrow$ & MRR $\uparrow$ & Hits@1 $\uparrow$ & Hits@10 $\uparrow$\\
\midrule
SISOKR & 2230.48 & 0.015 & 0.4\% & 2.8\%\\
SciBERT Regression & 344.53 & 0.298 & 16.8\% & 56.9\%\\
\hdashline
CMAM - MLP  & 23.74 &  0.513 & 34.9\% & 84.2\%\\
CMAM - GCN   &  24.11 & 0.495 & 33.2\% & 82.5\%\\
CMAM - Ensemble (MLP)  &  17.92 & 0.562 &39.8\%  & 87.6\% \\
CMAM - Ensemble (GCN)  & 20.48 &  0.551 & 39.0\% &  87.0\%\\
CMAM - Ensemble (MLP + GCN)  & \textbf{16.28} & 0.597 & 44.2\% & \textbf{88.7}\%  \\
\midrule
DSOKR - SubSample Sketch &  82.92 & 0.624  & 48.2\% & 87.4\% \\
DSOKR - Gaussian Sketch &  91.19 & 0.630 & 49.0\% & 87.5\% \\
DSOKR - Ensemble (SubSample Sketch) & 76.43  & \textbf{0.642} & \textbf{51.0}\% & 88.2\% \\
DSOKR - Ensemble (Gaussian Sketch) &  81.70 & \textbf{0.642} & 50.5\% & 87.9\% \\
DSOKR - Ensemble (SubSample + Gaussian) & 76.87 & 0.640 & 50.0\% & 88.3\%\\
\bottomrule
\end{tabular}
\end{center}
\label{tab:expr_t2m_complete}
\end{table*}

\end{document}